\documentclass{article}

\usepackage{microtype}
\usepackage{graphicx}
\usepackage{subcaption}
\usepackage{booktabs} % for professional tables

\usepackage{hyperref}

\usepackage[preprint]{icml2026}

\usepackage{amsmath}
\usepackage{amssymb}
\usepackage{mathtools}
\usepackage{amsthm}

\usepackage[capitalize,noabbrev]{cleveref}

\theoremstyle{plain}
\newtheorem{theorem}{Theorem}[section]

\theoremstyle{definition}

\theoremstyle{remark}

\usepackage[textsize=tiny]{todonotes}

\icmltitlerunning{Direct Kolen–Pollack Predictive Coding}

\begin{document}

\twocolumn[
  \icmltitle{Accelerated Predictive Coding Networks via \\ Direct Kolen–Pollack Feedback Alignment}



  \icmlsetsymbol{equal}{*}

  \begin{icmlauthorlist}
    \icmlauthor{Davide Casnici}{tud}
    \icmlauthor{Martin Lefebvre}{tud}
    \icmlauthor{Justin Dauwels}{tud,equal}
    \icmlauthor{Charlotte Frenkel}{tud,equal}
  \end{icmlauthorlist}

  \icmlaffiliation{tud}{Department of Microelectronics, Delft University of Technology, Delft, Netherlands}
  % This command generates the footnote for equal contribution
  \icmlsetsymbol{equal}{$\dagger$}

  \icmlcorrespondingauthor{Davide Casnici}{d.casnici@tudelft.nl}
  \icmlkeywords{Machine Learning, Predictive Coding, Feedback Alignment, Biologically Inspired Learning}

  \vskip 0.3in
]



%\printAffiliationsAndNotice{}  % no special notice (required even if empty)
\printAffiliationsAndNotice{\icmlEqualContribution}

\begin{abstract}
Predictive coding (PC) is a biologically inspired algorithm for training neural networks that relies only on local updates, allowing parallel learning across layers. However, practical implementations face two key limitations: error signals must still propagate from the output to early layers through multiple inference-phase steps, and feedback decays exponentially during this process, leading to vanishing updates in early layers. We propose direct Kolen–Pollack predictive coding (DKP-PC), which simultaneously addresses both feedback delay and exponential decay, yielding a more efficient and scalable variant of PC while preserving update locality. Leveraging direct feedback alignment and direct Kolen–Pollack algorithms, DKP-PC introduces learnable feedback connections from the output layer to all hidden layers, establishing a direct pathway for error transmission. This yields an algorithm that reduces the theoretical error propagation time complexity from $\mathcal{O}(L)$, with $L$ being the network depth, to $\mathcal{O}(1)$, removing depth-dependent delay in error signals. Moreover, empirical results demonstrate that DKP-PC achieves performance at least comparable to, and often exceeding, that of standard PC, while offering improved latency and computational performance, supporting its potential for custom hardware-efficient implementations.
\end{abstract}

\section{Introduction}
\label{intro}

Major advances in artificial intelligence, from image recognition \citep{lecun2002gradient, krizhevsky2017imagenet, alom2018history} to image generation \citep{kingma2013auto, parmar2018image, goodfellow2020generative} and natural language processing \citep{hochreiter1997long, vaswani2017attention, beck2024xlstm}, have all been enabled by backpropagation of error (BP), the fundamental algorithm underlying the training of artificial neural networks (ANNs) \citep{linnainmaa1970representation, rumelhart1986learning, WerbosBP}. However, several studies have put into question the plausibility of its direct implementation in biological neural systems \citep{grossberg1987competitive,lillicrap2016random,lillicrap2019backpropagation,whittington2019theories,ellenberger2024backpropagation}. Two primary concerns come from (i) the reliance on a global error signal that must be propagated backward and sequentially through the network hierarchy, thereby blocking parameter updates, and (ii) early layers depending on error signals generated by nodes not directly connected to them. These biological plausibility issues of BP are commonly referred to as \textit{update locking} and \textit{non-locality} \citep{nokland2016direct,frenkel2021learning,ororbia2023brain}. Importantly, they also lead to inefficiencies in hardware implementations, imposing memory and latency overheads \citep{mostafa2018deep,frenkel2023bottom}.

Predictive coding (PC), originally introduced as a model of the visual cortex in the human brain \citep{rao1999predictive, huang2011predictive}, is emerging as a more biologically plausible alternative to BP, alleviating its update-locking and non-locality limitations \citep{millidge2022predictivefuture, salvatori2023brain}. Its framework is grounded in Bayesian inference under the Free Energy Principle \citep{friston2005theory, friston2006free, friston2009predictive}, providing a rigorous mathematical foundation with connections to information theory \citep{elias1955predictive, elias2003predictive} and energy-based models \citep{millidge2022backpropagation, millidge2022theoretical}. Rather than minimizing a global error signal, PC minimizes the network’s \textit{variational free energy} (FE), defined as the sum of layer-wise squared errors between each layer’s activity and its incoming prediction. Unlike BP, where weights are directly updated, PC learning has two phases. In the \textit{inference phase}, neural activity is updated to minimize the FE, and in the \textit{learning phase}, weights are updated based on the optimized neural activity. However, while this framework yields local and layer-wise update rules, the error in PC is still generated at the output and must propagate backward during inference. This error-delay issue undermines PC’s theoretical advantages over BP by requiring a minimum number of inference steps proportional to the network depth, thereby limiting its efficiency and benefits for custom hardware implementations \citep{zahid2023predictive}. Moreover, the delayed error decays exponentially with depth, yielding vanishing updates in early layers \citep{pinchetti2024benchmarking, goemaere2025error}.

To address these limitations, we propose to propagate error information from the output layer to all hidden layers, yielding an instantaneous error term across the hierarchy. We thus build on feedback alignment methods \citep{lillicrap2014random}. Direct feedback alignment (DFA) \citep{nokland2016direct} uses random direct feedback connections to deliver error signals from the output to all hidden layers, avoiding both error delay and decay. However, DFA scales poorly, especially in deep convolutional networks. Direct Kolen-Pollack (DKP) improves DFA by learning the feedback matrices \citep{webster2020learning}, incorporating learning rules inspired by the Kolen-Pollack (KP) algorithm \citep{kolen1994backpropagation, akrout2019deep}, thereby enhancing performance while preserving locality. Figure~\ref{fig:circuits} illustrates these frameworks and shows how our proposed direct KP predictive coding (DKP-PC) integrates primitives of both PC and DKP.

Our contributions are summarized as follows:
\begin{enumerate}
    \item We extend the empirical analysis of \citet{webster2020learning} by providing a mathematical motivation for why DKP achieves closer alignment with BP than standard DFA. This novel view further supports the integration of DKP within the PC framework as an efficient preliminary weight update to generate an instantaneous error term at every layer.
    
    \item We introduce the DKP-PC algorithm, which simultaneously mitigates the feedback error delay and exponential decay limitations of PC while preserving locality. This, for the first time, enables full parallelization in PC networks regardless of batch size. We further discuss how our proposed PC variant achieves a time complexity of $\mathcal{O}(1)$, compared to $\mathcal{O}(L)$ for standard PC, with $L$ being the network depth.

    \item We provide a theoretical and empirical analysis of the synergy between DKP-PC components, demonstrating how the PC neural activity update, under the DKP regime, leads to improved feedback update and, ultimately, to better and more stable gradient alignment with BP compared to standard DKP.
    
    \item We empirically demonstrate that DKP-PC performs on par with, or outperforms, both DKP and PC, benchmarking them across fully-connected and convolutional networks up to VGG-9 on Tiny ImageNet. We further assess DKP-PC’s computational efficiency, showing that it consistently achieves more than a $60\%$ reduction in training time for both VGG-7 and VGG-9, compared to standard PC.
\end{enumerate}

\begin{figure*}[!t]
\begin{center}
\includegraphics[width=0.9\textwidth]{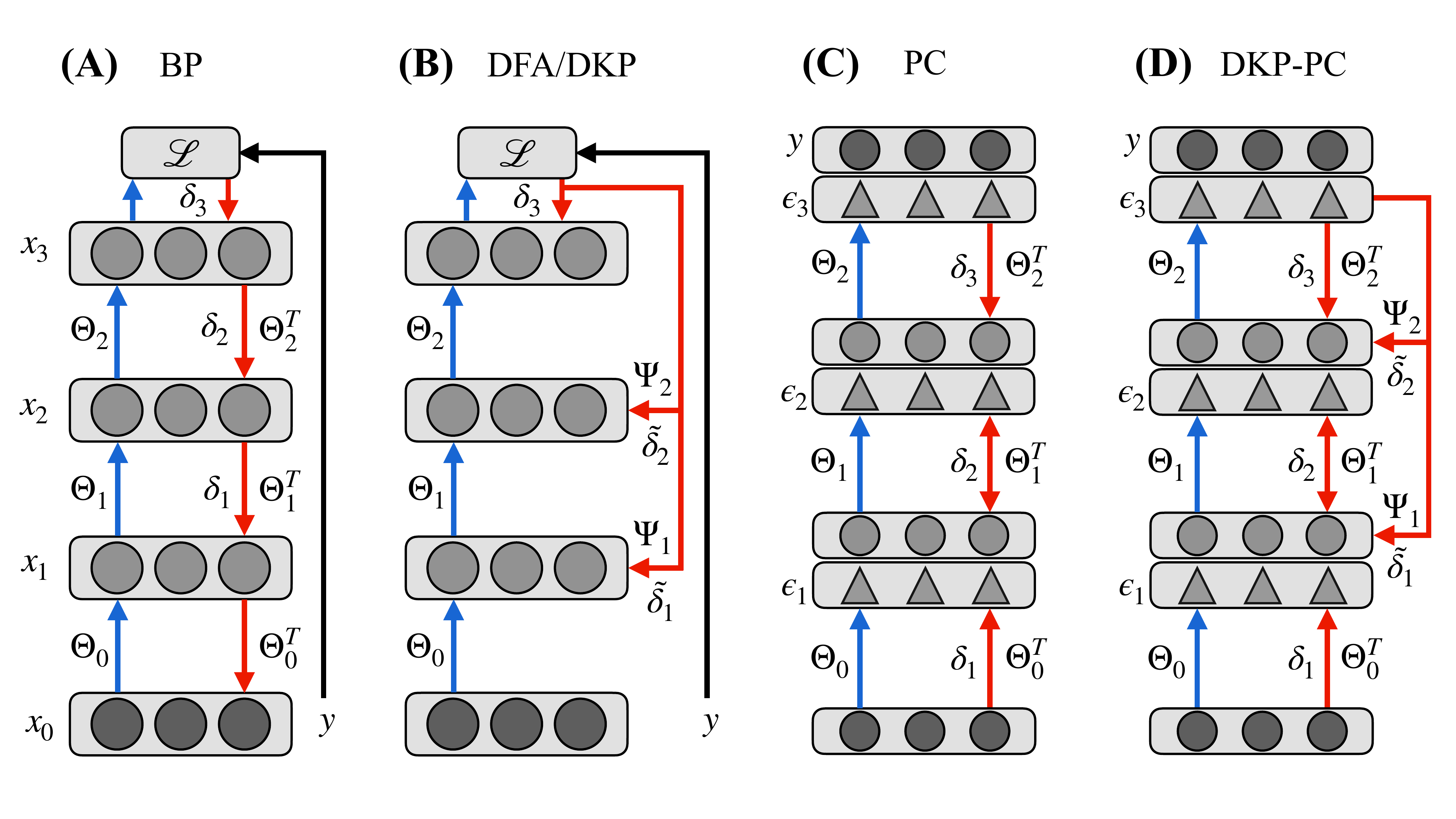}
\end{center}
\caption{DKP-PC embeds DKP within the PC framework to address the error feedback delay and exponential decay issues of PC. Blue arrows represent forward connections, red arrows represent feedback connections. Neural activities are shown as gray circles, with clamped values in darker gray; $x_0$ denotes the input and $y$ the target. $\mathcal{L}$ is the loss function, $\delta_\ell$ are the BP errors, $\tilde{\delta}_\ell$ their approximations, and $\epsilon_\ell$ the PC error neurons, represented as triangles. (A) BP propagates the global error sequentially. (B) DFA and DKP propagate the error directly from the output to each layer. (C) PC minimizes local errors through an inference phase, followed by a learning phase that updates weights. (D) DKP-PC employs direct feedback to deliver instantaneous error signals to every layer, accelerating error propagation during the PC inference phase while preserving the locality of weight updates.}
\label{fig:circuits}
\end{figure*}

\section{Background}
\label{background}

In this section, we review from a mathematical perspective the core concepts of BP, DFA/DKP, and PC, which form the basis of our DKP-PC algorithm.

\subsection{Backpropagation}
BP enables recursive and efficient computation of parameter gradients by applying the chain rule of calculus to propagate error derivatives from the output layer back through the network \citep{linnainmaa1970representation, rumelhart1986learning}. Let us consider a neural network as shown in Figure~\ref{fig:circuits}(A), where each
layer $\ell \in \{0,\dots,L\}$ is associated with an activity error $x_\ell \in \mathbb{R}^{d_\ell}$, where $x_0$ denotes the network's input and $x_L$ the network's output, and $d_\ell$ is the number of neurons in layer $\ell$. The forward pass is defined recursively as
\begin{equation}
    z_\ell = \Theta_{\ell-1}x_{\ell-1}, \quad x_\ell = f(z_\ell), \quad 1 \leq \ell \leq L,
\label{eq:bp_forward}
\end{equation}
where $\Theta_\ell \in \mathbb{R}^{d_{\ell+1}\times d_\ell}$ is the synaptic weight matrix mapping activity from layer $\ell$ to layer $\ell+1$, and $f : \mathbb{R}^{d_\ell} \rightarrow \mathbb{R}^{d_\ell}$ is typically an element-wise non-linear activation function. The output error is then
expressed in terms of the least-squared error (LSE)
\begin{equation}
\mathcal{L} = \frac{1}{2} \| x_L - y \|^2_2,
\label{eq:bp_lse}
\end{equation}
where $y \in \mathbb{R}^{d_L}$ is the target vector. Applying the chain rule, the recursively backpropagated errors $ \frac{\partial \mathcal{L}}{\partial z_{\ell}} = \delta_\ell \in \mathbb{R}^{d_\ell}$ are
\begin{equation}
\delta_\ell =
\begin{cases}
x_L - y & \text{if } \ell = L, \\
f'(z_{\ell}) \odot (\Theta_\ell^\top \delta_{\ell+1}) & \text{otherwise},
\end{cases}
\label{eq:bp_recursive_next}
\end{equation}
where $\odot$ denotes the Hadamard product between the activation derivative $f'(z_\ell)$ and the error term\footnote{Formally, $f'(z_\ell) = \frac{\partial x_\ell}{\partial z_\ell} \in \mathbb{R}^{d_\ell \times d_\ell}$ is the Jacobian matrix of the activation function. However, for element-wise activations, this Jacobian is diagonal with entries $f'(z_\ell)$, so the matrix-vector multiplication simplifies to the Hadamard product.}. The weights are then updated according to
\begin{equation}
\Delta\Theta_\ell = -\alpha \frac{\partial \mathcal{L}}{\partial \Theta_\ell} = -\alpha \big( \delta_{\ell+1}  x_\ell^\top \big),
\label{eq:bp_weights_update_rule}
\end{equation}
where $\alpha \in \left(0,1\right)$ is the weight learning rate.

\subsection{Direct Kolen-Pollack Feedback Alignment}
DFA explicitly addresses the challenge of iteratively backpropagating error information in BP, yielding a local and more biologically plausible algorithm \citep{nokland2016direct}. To achieve this, DFA introduces random matrices $\Psi_\ell \in \mathbb{R}^{d_\ell \times d_L}$ that connect the output layer directly to each hidden layer in the network, as illustrated in Figure~\ref{fig:circuits}(B). These matrices enable the direct propagation of the error signal $\delta_L$ generated at the output layer, avoiding the iterative backward propagation in Eq.~\eqref{eq:bp_recursive_next}. The error is projected directly into each hidden layer according to
\begin{equation}
\tilde{\delta}_\ell = f'(z_\ell) \odot (\Psi_\ell \delta_L),
\label{dfa_error}
\end{equation}
resulting in the following local weight update rule:
\begin{equation}
\Delta \tilde{\Theta}_\ell = -\alpha \big( \tilde{\delta}_{\ell+1}  x_\ell^\top \big).
\label{dfa_update}
\end{equation}
DKP extends DFA by introducing a local learning rule for the feedback matrices $\Psi_\ell$ \citep{webster2020learning}, inspired by the KP algorithm \citep{kolen1994backpropagation, akrout2019deep}. In contrast to DFA where the random matrices $\Psi_\ell$ are kept fixed, DKP updates them according to the following learning rule:
\begin{equation}
\Delta \Psi_\ell = -\alpha \big(x_{\ell} \delta_L^\top \big),
\label{dkp_update_rule}
\end{equation}
where the synaptic plasticity of $\Psi_\ell$ depends only on the connected hidden layer's activity and the output error signal. Note that, as this update depends solely on local information and is decoupled across layers, it can be parallelized without update locking. In Appendix~\ref{app:dkp}, we extend the empirical analysis of \citet{webster2020learning} by providing a mathematical demonstration that, under linear assumptions, the feedback matrices converge to values that incorporate the recursive chain in Eq.~\eqref{eq:bp_recursive_next}, despite the dimensionality mismatch between $\Psi_\ell$ and $\Theta_\ell$. This offers a new theoretical perspective explaining why DKP aligns more closely with BP compared to DFA, showing that it converges to a recursive Moore–Penrose pseudoinverse chain of the forward weights:
\begin{equation}
\Psi_{L-\ell} =
\begin{cases}
    \Theta_{L-\ell}^\top & \text{if } \ell = 1,\\
    \Theta_{L-\ell}^\top \left( \Psi_{L-\ell+1}^\top \right)^{+} & \text{if } 1 < \ell < L.
\end{cases}
\label{eq:dkp_recursive_chain_mpp}
\end{equation}

\subsection{Predictive Coding}
\label{pc}
PC models the brain as a Bayesian hierarchical generative model, in which neural activities encode the causes of sensory stimuli as Gaussian latent variables \citep{rao1999predictive, friston2005theory, friston2009predictive}. Specifically, the activity of layer $\ell$ is represented by a latent variable $x_\ell  \sim \mathcal{N}(\mu_\ell, \Sigma_\ell)$, where $\mu_\ell \in \mathbb{R}^{d_\ell}$ and $\Sigma_\ell \in \mathbb{R}^{d_\ell \times d_\ell}$ denote the mean and covariance of the layer, respectively. As done by several works in literature, we assume the generative model’s covariance to be fixed to the identity matrix $\Sigma_\ell = I$ \citep{pinchetti2022predictive, millidge2022predictivefuture, salvatori2024stablefastfullyautomatic}. The distribution's mean $\mu_\ell$ is parametrized by the previous layer's state through the synaptic weights $\Theta_\ell \in \mathbb{R}^{d_{\ell} \times d_{\ell-1}}$ connecting them, according to the relation $\mu_{\ell} = f(\Theta_{\ell-1} x_{\ell-1})$, where $f:\mathbb{R}^{d_\ell} \to \mathbb{R}^{d_\ell}$ is a non-linear mapping. The joint generative model over all $L+1$ latent-variables layers is
\begin{equation}
\begin{aligned}
&p_{\Theta}(x_0, \dots, x_L) =\\
&\mathcal{N}\big(x_0; \mu_0, \Sigma_0\big) \prod_{\ell=1}^{L} \mathcal{N}\big(x_\ell; \mu_\ell, \Sigma_\ell\big),
\end{aligned}
\label{eq:generative_model_combined}
\end{equation}
where $x_0$ and $x_{L}$ are clamped respectively to the input and target vectors, in classification settings. The exact posterior distribution inference $p(x_0, \dots, x_{L-1} \mid x_{L})$ is generally intractable \citep{friston2005theory}, so PC employs variational inference to approximate the latter with a tractable distribution $q(x_0, \dots, x_L)$, defined as 
\begin{equation}
    q(x_0, \dots, x_{L-1}) = \prod_{\ell=0}^{L-1}q(x_\ell),
\end{equation}
where $q(x_\ell) \in \mathbb{R}^{d_\ell}$ is the variational distribution over the layer $\ell$. Following PC literature for classification tasks, we model the variational posterior as a Dirac delta centered on parameter $\phi_\ell$:
\begin{equation}
q(x_\ell; \phi_\ell) = \delta(x_\ell - \phi_\ell),
\end{equation}
where $\phi_\ell = \mu_\ell$ \citep{bogacz2017tutorial, millidge2021predictive, pinchetti2022predictive, salvatori2024stablefastfullyautomatic}.
This formulation provides a deterministic approximation of the latent variables' mode. Variational inference reduces the problem of maximizing Eq.~\eqref{eq:generative_model_combined} to minimizing the Kullback–Leibler divergence between the variational and true posterior distributions $D_{KL}\left(q\,\|\,p\right)$. Equivalently, this corresponds to minimizing the variational FE, which constitutes an upper bound on \(D_{KL}(q \,\|\, p)\) \citep{friston2009predictive}. Under the assumptions of identity covariance matrices for the generative model and a Dirac delta variational posterior, the FE can be expressed as
\begin{equation}
F = \frac{1}{2} \sum_{\ell=1}^{L} \| \epsilon_\ell \|^2_2,
\label{eq:fe_sum_errors}
\end{equation}
where the prediction errors $\epsilon_\ell \in \mathbb{R}^{d_\ell}$ at layer $\ell$ are defined as
\begin{equation}
\epsilon_\ell = \phi_\ell - f(\Theta_{\ell-1} \phi_{\ell-1}),
\label{eq:pc_errors}
\end{equation}
and are considered as dedicated units, represented by triangles in Figure~\ref{fig:circuits}(C). In prediction tasks, $\phi_0$ is not inferred, as the input layer is clamped to the input vector, and the network is initialized via a forward pass as in Eq.~\eqref{eq:bp_forward}. After this initialization, the output layer $\phi_L$ is clamped to the target vector, so it is also not inferred. Minimizing Eq.~\eqref{eq:fe_sum_errors} results in local updates for both neurons and weights, enabling layer-wise learning. 

PC learning is divided into two phases: the inference and the learning phases. During the inference phase, Eq.~\eqref{eq:fe_sum_errors} is optimized with respect to the variational parameters $\phi_\ell$, updating the neural activities iteratively, whereas during the learning phase, the synaptic parameters $\Theta_\ell$ are updated using the resulting neural configuration to minimize the same objective. This yields the following neural activity update:
\begin{equation}
    \Delta \phi_\ell = -\gamma \frac{\partial F}{\partial \phi_\ell} = \gamma \Big( (f'(\Theta_\ell\phi_\ell) \odot \Theta_\ell)^\top \, \epsilon_{\ell+1} -\,\epsilon_\ell \Big),
\label{eq:neural_act_learn_rule}
\end{equation}
where $\gamma \in \left(0,1\right)$ is the neural activity learning rate. The rule is local, as the activity of $\phi_\ell$ is influenced only by the adjacent error nodes $\epsilon_\ell$ and $\epsilon_{\ell+1}$. After the inference phase is completed, the synaptic connections are updated using the final neural activity configuration $\phi_\ell^*$:
\begin{equation}
\Delta \Theta_\ell = -\alpha \frac{\partial F}{\partial \Theta_\ell} = \alpha \, \big( f'(\Theta_\ell\phi_\ell^*) \odot \, \epsilon_{\ell+1} \, \phi_\ell^{* \top} \big),
\label{eq:weight_update_rule}
\end{equation}
where $\alpha \in \left(0,1\right)$ is the weight learning rate. The weight update is local, as it depends only on the error neurons of the next layer $\epsilon_{\ell+1}$, and on the optimized neural activity of the current layer $\phi_\ell^*$.

\section{Methodology}
\label{methodology}

In this section, we first introduce the issues of error delay and error decay in PC, and subsequently demonstrate how the proposed DKP-PC algorithm provides a unified solution to both of them.

\subsection{Feedback error delay and decay}
\label{errordelaydecay}
Let us consider a forward-initialized PC network with $L+1$ latent variables layers, whose neural dynamics evolve in discrete time steps $t \in \mathbb{N}_0$ according to Eq.~\eqref{eq:neural_act_learn_rule}, explicitly denoting the time dependence of neural activities $\phi_\ell(t)$ and prediction errors $\epsilon_\ell(t)$ during the inference phase.

\emph{Error propagation delay} -- At the initial time $t=0$, the neural activity of each layer corresponds to the prediction from the previous one, resulting in null error values at every layer in the network:
\begin{equation}
\begin{aligned}
&\phi_\ell(0) = f(\Theta_{\ell-1} \phi_{\ell-1}(0)) \\ &\implies \quad\epsilon_\ell(0) = \phi_\ell(0) - f(\Theta_{\ell-1} \phi_{\ell-1}(0)) = 0.
\end{aligned}
\label{eq:zero-error-pc}
\end{equation}
The network is thus at equilibrium, since the FE in Eq.~\eqref{eq:fe_sum_errors} is minimized \citep{whittington2017approximation}. Assuming an incorrect prediction, clamping the target vector $y$ to the output layer at $t=0$ induces a non-zero error at the final layer $\epsilon_L(0) \neq 0$, as shown in Figure~\ref{fig:error_delay_decay}(A). Since each layer updates its activity based on its own and the subsequent layer’s error (see Eq.~\eqref{eq:neural_act_learn_rule}), the error propagates backward at most one layer per time step. Thus, the error takes $\hat{t} = L- \ell$ inference-phase steps to reach layer $\ell$ \citep{zahid2023predictive} (see theorem and proof in Appendix~\ref{app:error_delay}). 

\emph{Error exponential decay} -- The error at optimization time $\hat{t} = L - \ell$, denoted $\epsilon_\ell(\hat{t})$, decays exponentially as it propagates backwards through the network, as shown in Figure~\ref{fig:error_delay_decay}(A). The decay rate is determined by both the learning rate and the distance from the output layer \citep{goemaere2025error}, and the squared $\ell_2$-norm $\|\epsilon_\ell (\hat{t})\|^2_2$ is upper-bounded by a quantity $\propto \gamma^{2(L-\ell)}$ (see theorem and proof in Appendix~\ref{app:error_decay}).

Both issues originate from the fact that the error is generated only at the output layer, with the delay arising from its layer-by-layer propagation through the network and the exponential decay resulting from its progressive reduction by the neural activity learning rate at each iteration.

\begin{figure*}[!t]
\begin{center}
\includegraphics[width=0.9\textwidth]{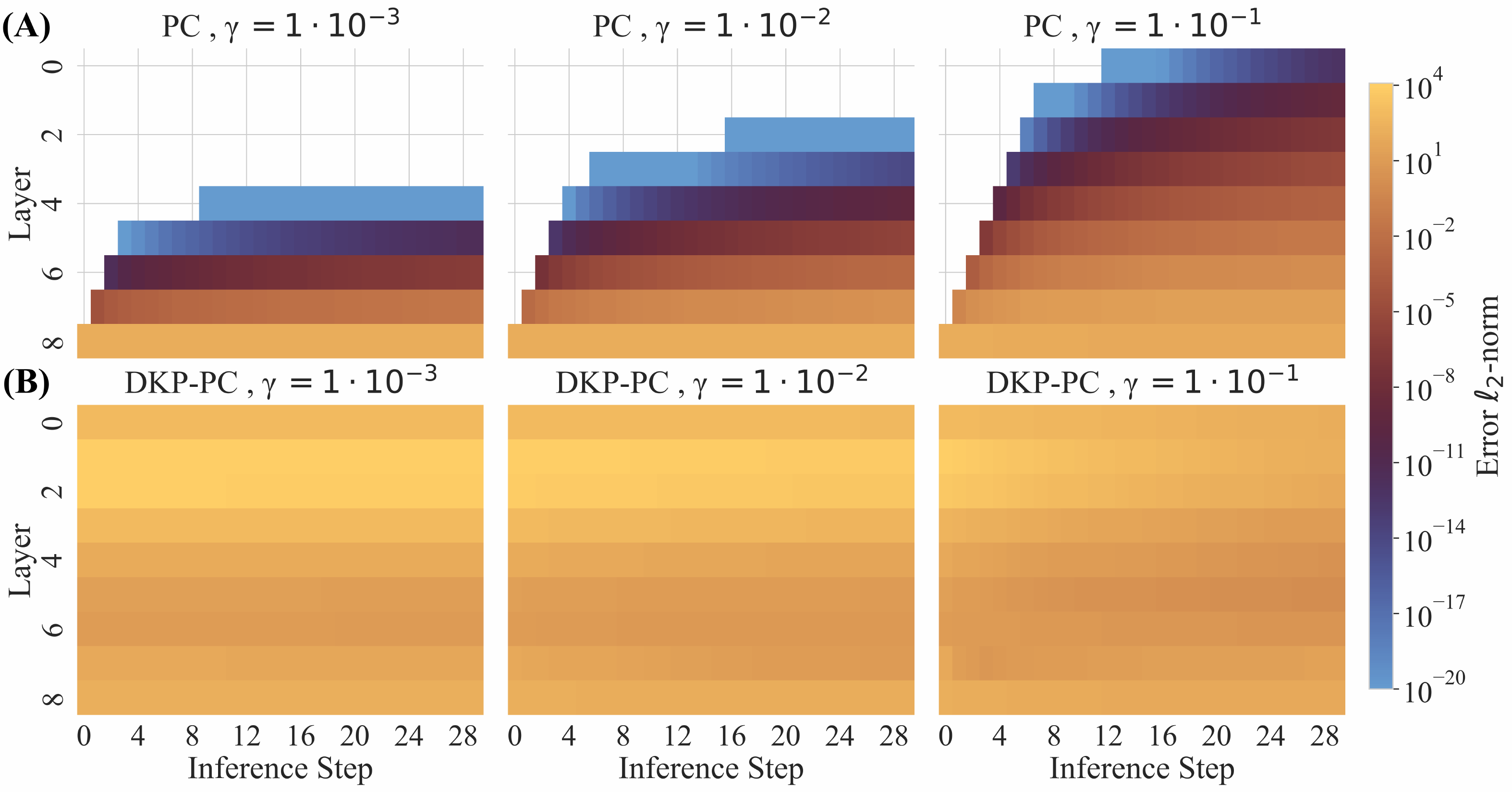}
\end{center}
\caption{Error propagation in PC (A) and DKP-PC (B) during the inference phase of a VGG-9 network trained on a single CIFAR-10 batch, at different magnitudes of the neural activity learning rate $\gamma$. In (A), PC exhibits both an error decay problem, where the error magnitude decreases exponentially with network depth, and an error delay problem, as the error signal flows through the network sequentially, undermining the theoretical parallelism. White colour represents values equal to zero or below the numerical precision. In (B), DKP-PC mitigates both issues, generating a more uniform error signal across all layers at the start of neural activity optimization.}
\label{fig:error_delay_decay}
\end{figure*}

\subsection{Direct Kolen-Pollack Predictive Coding}
We propose to introduce learnable feedback connections from the output layer to each hidden layer $\Psi_\ell \in \mathbb{R}^{d_\ell \times d_L} \, , \, \forall \, \ell \in \{1, \cdots, L-1\}$, to enable a preliminary update of the forward weights (Figure~\ref{fig:circuits}(D)). This approach, inspired by the DKP algorithm, not only induces approximate alignment with BP toward minimizing the output error (Appendix~\ref{app:dkp}), but also ensures that a non-zero error term is generated at every layer at the beginning of the inference phase, since the condition in Eq.~\eqref{eq:zero-error-pc} no longer holds (Figure~\ref{fig:error_delay_decay}(B)).

The resulting algorithm, denoted as DKP-PC, not only solves the error propagation delay and exponential decay issues of PC, but also allows speeding up the inference phase of PC. Indeed, after the error is directly propagated, we empirically show that a single inference-phase step is sufficient to match or even surpass the performance of standard PC, which typically requires a number of steps at least equal to, and often exceeding, the network depth \citep{pinchetti2024benchmarking}. We provide the pseudocode of DKP-PC in Algorithm~\ref{alg:dkp-pc}, and formally outline the main steps below.

\begin{algorithm}[tb]
\caption{Direct Kolen-Pollack Predictive Coding}
\label{alg:dkp-pc}
\begin{algorithmic}%[1]

\FOR{each $(x,y) \in \mathcal{D}$}

    \STATE{0) \textit{Forward initialization} \hfill \textbf{Sequential}} 
    \STATE $\phi_0 \leftarrow x$
    \FOR{$\ell = 1$ {\bfseries to} $L-1$}
        \STATE $\phi_\ell \leftarrow f(\Theta_{\ell-1}\phi_{\ell-1})$ 
    \ENDFOR
    \STATE $\phi_L \leftarrow y$
    \STATE $\epsilon_L \leftarrow y - f(\Theta_{L-1}\phi_{L-1})$

    \STATE{1) \textit{Direct Feedback Alignment update} \hfill \textbf{Parallel}}
    \FOR{$\ell = 0$ {\bfseries to} $L-1$} 
        \STATE $\Theta_\ell \leftarrow \Theta_\ell + \alpha \big( f'(\Theta_\ell \phi_\ell) \odot (\Psi_{\ell+1}\epsilon_L)\phi_\ell^\top \big)$
    \ENDFOR

    \STATE{2) \textit{Inference phase} \hfill \textbf{Parallel}}
    \FOR{$t=0$ {\bfseries to} $T$}
        \FOR{$\ell = 1$ {\bfseries to} $L$, $\ell < L$ for $\phi_\ell$} 
            \STATE $\epsilon_\ell \leftarrow \phi_\ell - f(\Theta_{\ell-1}\phi_{\ell-1})$
            \STATE $\phi_\ell \leftarrow \phi_\ell - \gamma \frac{\partial F}{\partial \phi_\ell}$
        \ENDFOR
    \ENDFOR

    \STATE{3) \textit{Learning phase} \hfill \textbf{Parallel}}
    \FOR{$\ell = 0$ {\bfseries to} $L-1$, $\ell > 0$ for $\Psi_\ell$}
        \STATE $\Theta_\ell \leftarrow \Theta_\ell - \alpha \frac{\partial F}{\partial \Theta_\ell}$
        \STATE $\Psi_\ell \leftarrow \Psi_\ell + \alpha \, \phi_{\ell}\epsilon_L^\top$
    \ENDFOR
\ENDFOR

\label{alg:dkppc-alg}
\end{algorithmic}
\label{alg:dkppc-alg}
\end{algorithm}

\textit{Direct feedback alignment update} --  After the forward initialization of the network, assuming a non-zero error at the output layer $\epsilon_L$, we perturb the equilibrium by taking a first weight update according to Eq.~\eqref{dfa_update}, using PC's last layer error neurons:
\begin{equation}
\begin{aligned}
    \tilde{\Theta}_\ell &= \Theta_\ell + \Delta\tilde{\Theta}_\ell \\
    &= \Theta_\ell -\alpha \big( f'(\Theta_\ell \phi_\ell) \odot (\Psi_{\ell+1} \delta_L) \phi_\ell^\top \big)\\
    &= \Theta_\ell +\alpha \big( f'(\Theta_\ell \phi_\ell) \odot (\Psi_{\ell+1} \epsilon_L) \phi_\ell^\top \big),
\end{aligned}
\label{eq:dkppc-preliminary-update}
\end{equation}
which can be performed in parallel for each layer, as there is no recursive dependence.

\textit{Inference phase} -- After this update, an error term is generated at every layer since the first time step:
\begin{equation}
    \|\epsilon_\ell(0)\|_2^2 = \| \phi_\ell(0) - f(\tilde{\Theta}_{\ell-1} \phi_{\ell-1}(0))\|_2^2 > 0,
\end{equation}
where $\phi_\ell(0) = f(\Theta_{\ell-1} \phi_{\ell-1}(0))$ after forward initialization. This provides a non-zero error term instantaneously in every layer-wise component of the FE in Eq.~\eqref{eq:fe_sum_errors}. Consequently, every layer can independently update its neural activity according to Eq.~\eqref{eq:neural_act_learn_rule}, without performing null updates while waiting for the propagation of the error from the last layer. Thus, the neural activity optimization performed takes the following form:
\begin{equation}
        \Delta \phi_\ell = \gamma \Big( (J_f(\tilde{\Theta}_\ell \phi_\ell) \, \tilde{\Theta}_\ell)^\top \, \epsilon_{\ell+1} -\,\epsilon_\ell \Big).
\label{eq:dkppc_neural_update}
\end{equation}
Furthermore, in contrast to standard PC, the neural activity now incorporates the information injected into the forward weights by the preliminary DKP update. In Appendix~\ref{app:dkp-pc-theoretical}, we show theoretically in Eq.~\eqref{eq:dkp-pc_neural_activity_update} that, under linear assumptions, this corresponds to enforcing alignment and regularization through the single-step neural activity update, which in turn improves the alignment of the forward and feedback weights, as further supported by empirical evidence in Appendix~\ref{app:dkp-pc-empirical}. Lastly, although the update in Eq.~\eqref{eq:dkppc_neural_update} is applied only once to maximize the acceleration of PC networks afforded by DKP-PC, our method is not limited to this setting. As discussed in Appendix~\ref{app:inference-phase}, DKP-PC can leverage the same trade-off as standard PC networks, further minimizing the network's FE to achieve higher classification accuracy by performing multiple inference steps, at the cost of increased training time.

\textit{Learning phase and DKP update} -- After this single local update, both feedforward and feedback weight matrices are updated, which can be fully parallelized. Starting from the feedforward weight matrices $\Theta_\ell$, the update rule is unchanged from the standard PC update in Eq.~\eqref{eq:weight_update_rule}, with the difference of using the neural activity resulting from Eq.~\eqref{eq:dkppc_neural_update}. Feedback weight matrices $\Psi_\ell$ now also incorporate PC's optimized neural activity:
\begin{equation}
\Delta \Psi_\ell = -\alpha \big(\phi_{\ell}^* \delta_L^\top \big) = \alpha \big(\phi_{\ell}^* \epsilon_L^\top \big).
\label{dkppc_feedback_update_rule}
\end{equation}
With DKP-PC, we introduce the first PC variant that fully enables its inherent theoretical parallelizability. Although its sequential formulation across stages may appear to challenge parallelizability and strict biological plausibility, each stage relies exclusively on locally available variables. This property enables full parallelization across layers without suffering from error delay and preserves computational locality throughout the entire learning process, a hallmark of biologically plausible learning. Consequently, the backward time complexity of the network is reduced from $\mathcal{O}(L)$ to $\mathcal{O}(1)$, as it no longer scales with the network depth $L$. Importantly, DKP-PC goes beyond the only other PC variant that addresses the error-delay issue, namely incremental predictive coding (iPC) \citep{salvatori2024stablefastfullyautomatic}. While iPC improves training stability by alternating between neural activity updates (Eq.~\eqref{eq:neural_act_learn_rule}) and feedforward weight updates (Eq.~\eqref{eq:weight_update_rule}), it only partially mitigates error delay and requires full-batch training to do so. In contrast, DKP-PC resolves this limitation independently of the batch size. Furthermore, Appendix~\ref{app:inference-phase} shows that DKP-PC admits an incremental formulation under a multiple inference phase steps regime.

\section{Results}
\label{results}
In this section, we assess DKP-PC against BP, DKP, PC, iPC, and center-nudging PC (CN-PC), in terms of classification performance and training speed. While our method achieves competitive accuracy, it significantly outperforms all other PC variants in terms of training speed, especially for larger networks and datasets.

\textit{Setup} -- We evaluate the scalability of DKP-PC from multi-layer perceptrons (MLPs) to VGG-like convolutional neural networks (CNNs) \citep{simonyan2014very}. For the MLP experiments, a three-layer architecture is evaluated on MNIST and Fashion-MNIST \citep{yann2010mnist, xiao2017fashion}. For the CNN experiments, we assess the performance of VGG-7 and VGG-9 on the CIFAR-10, CIFAR-100, and Tiny ImageNet datasets \citep{krizhevsky2009learning, le2015tiny}. For comparability with prior PC works and to facilitate future benchmarking, we employ the architectures reported by \citet{pinchetti2024benchmarking} in their discriminative mode experiments, and report their performance for PC, iPC, CN-PC and BP. Additional implementation details are reported in Appendix~\ref{app:train_specs}. All implementations are based on the PyTorch framework and are available on GitHub. \footnote{Link omitted to preserve the double-blind review process. The GitHub repository will also contain all hyperparameters for reproducibility.}

\begin{table*}[!ht]
\vspace{1em}
\caption{Test accuracy in \% (mean $\pm$ standard deviation) averaged over 5 random seeds. Results for PC, iPC, CN-PC and BP are taken from \citet{pinchetti2024benchmarking}. The best results among local algorithms are highlighted in bold.}
\vspace{1em}

\centering
\resizebox{\linewidth}{!}{
\begin{tabular}{l|ccccc|c}

\midrule
\% Accuracy & DKP & PC & iPC & CN-PC & \textbf{DKP-PC} & BP \\ 
\midrule

\small \textbf{MLP}  &  &  &  &  &  &  \\ 

MNIST & $98.03^{\pm{0.10}}$ &  $98.26^{\pm{0.04}}$&  $\mathbf{98.45^{\pm{0.09}}}$ & $98.23^{\pm 0.09}$ & $98.02^{\pm{0.09}}$ &  $98.29^{\pm{0.08}}$\\ 

FashionMNIST & $88.86^{\pm{0.13}}$ &  $89.58^{\pm{0.13}}$&  $\mathbf{89.90^{\pm{0.06}}}$ & $89.56^{\pm 0.05}$ & $89.42^{\pm{0.25}}$ &  $89.48^{\pm{0.07}}$\\ 
\midrule

\small \textbf{VGG-7}  &  &  &  &  &  & \\ 

CIFAR-10 & $77.98^{\pm{0.39}}$ & $81.91^{\pm{0.30}}$&  $80.15^{\pm{0.18}}$ & $\mathbf{88.40^{\pm 0.12}}$ & $82.36^{\pm{0.18}}$ & $89.91^{\pm{0.10}}$\\ 

CIFAR-100 (Top-1) & $36.96^{\pm{0.62}}$ & $37.52^{\pm{2.60}}$&  $43.99^{\pm{0.30}}$ & $\mathbf{64.76^{\pm 0.17}}$ & $50.42^{\pm{0.38}}$ & $65.36^{\pm{0.15}}$\\ 

CIFAR-100 (Top-5) & $64.93^{\pm{0.46}}$ & $66.73^{\pm{2.37}}$&  $73.23^{\pm{0.30}}$ & $\mathbf{84.65^{\pm 0.18}}$ & $77.24^{\pm{0.60}}$ & $84.41^{\pm{0.26}}$\\ 

\midrule

\small \textbf{VGG-9} &  &  &  &  &  &  \\ 

CIFAR-10  & $77.12^{\pm{0.33}}$ & $75.33^{\pm{0.25}}$& $79.02^{\pm{0.21}}$ & $\mathbf{87.19^{\pm 0.41}}$ & $81.95^{\pm{0.19}}$ & $90.02^{\pm{0.18}}$\\ 

CIFAR-100 (Top-1) & $46.07^{\pm{1.00}}$ & $39.57^{\pm{0.18}}$& $44.76^{\pm{0.40}}$ & $\mathbf{58.92^{\pm 1.61}}$ & $53.80^{\pm{0.64}}$ & $65.51^{\pm{0.23}}$\\ 

CIFAR-100 (Top-5) & $72.80^{\pm{1.06}}$ &  $66.90^{\pm{0.26}}$& $72.88^{\pm{0.29}}$ & $\mathbf{81.56^{\pm 0.63}}$ & $79.26^{\pm{0.63}}$ & $84.70^{\pm{0.28}}$\\ 

Tiny ImageNet (Top-1) & $29.61^{\pm{0.60}}$ & $21.78^{\pm{0.15}}$& $26.34^{\pm{0.03}}$ & $31.50^{\pm 0.70}$ & $\mathbf{35.04^{\pm{2.64}}}$ & $45.51^{\pm{0.15}}$\\ 

Tiny ImageNet (Top-5) & $53.03^{\pm{0.73}}$ &  $44.43^{\pm{0.09}}$& $50.48^{\pm{0.05}}$ & $54.67^{\pm 0.68}$ & $\mathbf{58.61^{\pm{3.12}}}$ & $65.62^{\pm{17}}$\\ 

\midrule

\end{tabular}
\label{tab:acc_comparison}
}
\end{table*}

\textit{Classification performance} -- The classification results are summarized in Table~\ref{tab:acc_comparison}. For the MLP architecture, all algorithms achieve comparable performance on both MNIST and FMNIST, with local algorithms even surpassing the test accuracy of BP. For VGG-7 and VGG-9 on CIFAR-10 and CIFAR-100, DKP-PC outperforms DKP, PC and iPC, achieving up to $14\%$ higher top-1 accuracy for VGG-9 on CIFAR-100 compared to standard PC, and $9\%$ higher top-1 accuracy than its more stable variant, iPC. However, CN-PC is the best local learning algorithm for all the settings mentioned so far. When moving to the Tiny ImageNet dataset, representing the most complex one evaluated in our experiments, we can see that DKP-PC outperforms all the local learning algorithms, achieving a final test accuracy of $35.04\%$, compared to $31.50\%$ for CN-PC. Furthermore, DKP-PC outperforms vanilla DKP in every setup evaluated, marking a gap of even $13\%$ top-1 accuracy for VGG-7 on CIFAR-100. Interestingly, by leveraging the complementary strengths of DKP and PC, DKP-PC consistently delivers higher accuracy than either method alone and substantially narrows the performance gap with BP, particularly in deeper architectures where local learning typically struggles.

\begin{table*}[!ht]
\vspace{1em}
\caption{Training time per epoch in seconds (mean $\pm$ standard deviation), averaged over 5 trials. All results were obtained via sequential execution of the algorithms.}
\vspace{1em}

\centering
\begin{tabular}{l|cccc|c}

\midrule
Seconds & DKP & PC & iPC & \textbf{DKP-PC} & BP \\ 
\midrule

\small \textbf{MLP}  &  &  &  &  &  \\ 

MNIST & $4.70^{\pm{0.10}}$ &  $4.71^{\pm{0.06}}$ &  $4.79^{\pm{0.09}}$ & $4.74^{\pm{0.09}}$ &  $4.70^{\pm{0.06}}$\\ 

FashionMNIST & $4.62^{\pm{0.06}}$ &  $4.77^{\pm{0.13}}$ &  $5.07^{\pm{0.13}}$ & $4.70^{\pm{0.14}}$ &  $4.62^{\pm{0.07}}$\\ 
\midrule

\small \textbf{VGG-7}  &  &  &  &  &  \\ 

CIFAR-10 & $7.21^{\pm{0.26}}$ & $31.39^{\pm{0.20}}$ &  $54.48^{\pm{0.12}}$ & $11.13^{\pm{0.09}}$ & $7.27^{\pm{0.19}}$\\ 

CIFAR-100 & $7.11^{\pm{0.06}}$ & $31.48^{\pm{0.17}}$ &  $54.69^{\pm{0.22}}$ & $11.67^{\pm{0.04}}$ & $7.15^{\pm{0.04}}$\\ 

\midrule

\small \textbf{VGG-9} &  &  &  &  &  \\ 

CIFAR-10  & $7.21^{\pm{0.09}}$ & $34.10^{\pm{0.15}}$ & $69.34^{\pm{0.10}}$ & $12.06^{\pm{0.03}}$ & $7.09^{\pm{0.08}}$\\ 

CIFAR-100 & $7.17^{\pm{0.06}}$ & $34.18^{\pm{0.04}}$ & $69.73^{\pm{0.09}}$ & $12.53^{\pm{0.01}}$ & $6.95^{\pm{0.05}}$\\ 

Tiny ImageNet & $35.37^{\pm{3.57}}$ & $158.48^{\pm{1.59}}$ & $303.14^{\pm{0.19}}$ & $54.10^{\pm{0.13}}$ & $38.27^{\pm{5.45}}$\\ 

\midrule

\end{tabular}
\label{tab:timing}
\end{table*}

\textit{Training speed} -- Table~\ref{tab:timing} shows the training times for one epoch, in seconds, for BP, DKP, PC, iPC, and DKP-PC, averaged over 5 trials, using the same experimental settings as in Table~\ref{tab:acc_comparison}. CN-PC is omitted since, after nudging the final layer, its training dynamics match those of standard PC (up to the update sign \citep{pinchetti2024benchmarking}). Importantly, DKP-PC requires only a single PC inference step to achieve the accuracies reported in Table~\ref{tab:acc_comparison}. In contrast, PC models typically need a number of inference steps equal to or larger than the network depth to reach the reported accuracy \citep{pinchetti2024benchmarking}. Consequently, in our timing evaluation, we set the number of PC inference steps equal to the network depth. Therefore, the reported PC training times should be interpreted as a lower bound. All measurements were performed using PyTorch on an NVIDIA RTX A6000 GPU where the parallelization opportunities offered by DKP and DKP-PC have not been leveraged, as they would require the use of custom CUDA kernels\footnote{A standard parallelization of DKP-PC in PyTorch introduces significant thread management and synchronization overhead, which cancels out the potential speedup.}. These models were thus executed sequentially, a setting in which BP naturally exploits highly-optimized hardware mapping and execution. Therefore, despite only highlighting speedup through reduced inference-phase steps and not through parallelization, the speedup achieved by DKP-PC compared to other PC algorithms is notable. Indeed, while on very small networks GPU and kernel overheads dominate the runtime, making algorithmic differences negligible, as depth increases the computational gap grows sharply. On the evaluated CNNs, DKP\text{-}PC delivers approximately an average training time reduction of $64\%$ compared to PC and $81\%$ compared to iPC. We elaborate further on the computational trade-offs of DKP-PC in Appendix \ref{app:computational_trade-offs}.

\section{Conclusion and Future Work}
We introduced DKP-PC, the first training algorithm that releases PC’s feedback error delay and exponential decay toward enabling fully parallelized, local learning. We evaluated its classification performance, training speed, and computational efficiency against BP, DKP, PC, and iPC. Our results show that, by accelerating PC with DKP, DKP-PC scales better than the evaluated local-learning algorithms, while exhibiting a substantial improvement in computational efficiency and training time compared to PC and its newer variants iPC and CN-PC. These results indicate that local learning rules can approach BP’s efficiency while narrowing the scalability gap, which is particularly relevant for neuromorphic computing and on-chip learning \citep{millidge2022predictivefuture,frenkel2023bottom}. Future work should focus on custom CUDA kernels to address the thread management and synchronization overheads of the current PyTorch implementation. Indeed, despite already a significant speed-up compared to PC, the training time of DKP-PC will still lag behind that of BP as long as parallelization opportunities are not fully exploited. Furthermore, as feedback matrices introduce memory overhead, sparsity and quantization of feedback weights should be explored, as incentivized by prior work \citep{crafton2019direct, han2019efficient}. Lastly, this novel combination of feedback-alignment methods and PC might pave the way for a new class of algorithms focused on exploiting the synergy between the two frameworks and leveraging their specific dynamics. An interesting research direction is to directly use the feedback information to perturb the neural activity dynamics, without relying on a preliminary weight update step, thereby outlining faster and more efficient local update rules for the neural activity dynamics. Future work could also focus on a tailored integration of DKP with advanced PC variants, such as nudging PC based on equilibrium propagation \citep{scellier2017equilibrium, scellier2023energy, pinchetti2024benchmarking}, combining their dynamics with the DKP learning rules for both forward and feedback weights. This integration could allow the different formulations to complement each other and further reduce the performance gap with BP.

%\section*{Acknowledgements}

\section*{Impact Statement}
This paper presents work whose goal is to advance the field of Machine
Learning. There are many potential societal consequences of our work, none
which we feel must be specifically highlighted here.

\bibliography{references}

@phdthesis{linnainmaa1970representation,
  title={The representation of the cumulative rounding error of an algorithm as a Taylor expansion of the local rounding errors (Doctoral dissertation, Master’s Thesis},
  author={Linnainmaa, S},
  year={1970},
  school={MA thesis. University of Helsinki}
}

@article{mostafa2018deep,
  title={Deep supervised learning using local errors},
  author={Mostafa, Hesham and Ramesh, Vishwajith and Cauwenberghs, Gert},
  journal={Frontiers in neuroscience},
  volume={12},
  pages={608},
  year={2018},
  publisher={Frontiers Media SA}
}

@article{ororbia2023brain,
  title={Brain-inspired machine intelligence: A survey of neurobiologically-plausible credit assignment},
  author={Ororbia, Alexander G},
  journal={arXiv preprint arXiv:2312.09257},
  year={2023}
}

@INPROCEEDINGS{WerbosBP,
  author={Werbos},
  booktitle={IEEE 1988 International Conference on Neural Networks}, 
  title={Backpropagation: past and future}, 
  year={1988},
  volume={},
  number={},
  pages={343-353 vol.1},
  doi={10.1109/ICNN.1988.23866}}

@article{bogacz2017tutorial,
  title={A tutorial on the free-energy framework for modelling perception and learning},
  author={Bogacz, Rafal},
  journal={Journal of mathematical psychology},
  volume={76},
  pages={198--211},
  year={2017},
  publisher={Elsevier}
}

@article{millidge2021predictive,
  title={Predictive coding: a theoretical and experimental review},
  author={Millidge, Beren and Seth, Anil and Buckley, Christopher L},
  journal={arXiv preprint arXiv:2107.12979},
  year={2021}
}

@article{saxe2013exact,
  title={Exact solutions to the nonlinear dynamics of learning in deep linear neural networks},
  author={Saxe, Andrew M and McClelland, James L and Ganguli, Surya},
  journal={arXiv preprint arXiv:1312.6120},
  year={2013}
}

@inproceedings{he2015delving,
  title={Delving deep into rectifiers: Surpassing human-level performance on imagenet classification},
  author={He, Kaiming and Zhang, Xiangyu and Ren, Shaoqing and Sun, Jian},
  booktitle={Proceedings of the IEEE international conference on computer vision},
  pages={1026--1034},
  year={2015}
}

@inproceedings{glorot2010understanding,
  title={Understanding the difficulty of training deep feedforward neural networks},
  author={Glorot, Xavier and Bengio, Yoshua},
  booktitle={Proceedings of the thirteenth international conference on artificial intelligence and statistics},
  pages={249--256},
  year={2010},
  organization={JMLR Workshop and Conference Proceedings}
}

@article{dozat2016incorporating,
  title={Incorporating nesterov momentum into adam},
  author={Dozat, Timothy},
  year={2016}
}

@article{adam2014method,
  title={A method for stochastic optimization},
  author={Adam, Kingma DP Ba J and others},
  journal={arXiv preprint arXiv:1412.6980},
  volume={1412},
  number={6},
  year={2014}
}

@article{loshchilov2017decoupled,
  title={Decoupled weight decay regularization},
  author={Loshchilov, Ilya and Hutter, Frank},
  journal={arXiv preprint arXiv:1711.05101},
  year={2017}
}

@article{simonyan2014very,
  title={Very deep convolutional networks for large-scale image recognition},
  author={Simonyan, Karen and Zisserman, Andrew},
  journal={arXiv preprint arXiv:1409.1556},
  year={2014}
}

@article{krizhevsky2009learning,
  title={Learning multiple layers of features from tiny images},
  author={Krizhevsky, Alex and Hinton, Geoffrey and others},
  year={2009},
  publisher={Toronto, ON, Canada}
}

@article{yann2010mnist,
  title={MNIST handwritten digit database},
  author={Yann, LeCun},
  journal={ATT Labs.},
  year={2010}
}

@article{xiao2017fashion,
  title={Fashion-mnist: a novel image dataset for benchmarking machine learning algorithms},
  author={Xiao, Han and Rasul, Kashif and Vollgraf, Roland},
  journal={arXiv preprint arXiv:1708.07747},
  year={2017}
}

@article{pinchetti2022predictive,
  title={Predictive coding beyond gaussian distributions},
  author={Pinchetti, Luca and Salvatori, Tommaso and Yordanov, Yordan and Millidge, Beren and Song, Yuhang and Lukasiewicz, Thomas},
  journal={arXiv preprint arXiv:2211.03481},
  year={2022}
}

@article{rumelhart1986learning,
  title={Learning representations by back-propagating errors},
  author={Rumelhart, David E and Hinton, Geoffrey E and Williams, Ronald J},
  journal={nature},
  volume={323},
  number={6088},
  pages={533--536},
  year={1986},
  publisher={Nature Publishing Group UK London}
}

@article{le2015tiny,
  title={Tiny imagenet visual recognition challenge},
  author={Le, Yann and Yang, Xuan},
  journal={CS 231N},
  volume={7},
  number={7},
  pages={3},
  year={2015}
}

@inproceedings{refinetti2021align,
  title={Align, then memorise: the dynamics of learning with feedback alignment},
  author={Refinetti, Maria and d’Ascoli, St{\'e}phane and Ohana, Ruben and Goldt, Sebastian},
  booktitle={International Conference on Machine Learning},
  pages={8925--8935},
  year={2021},
  organization={PMLR}
}

@article{scellier2017equilibrium,
  title={Equilibrium propagation: Bridging the gap between energy-based models and backpropagation},
  author={Scellier, Benjamin and Bengio, Yoshua},
  journal={Frontiers in computational neuroscience},
  volume={11},
  pages={24},
  year={2017},
  publisher={Frontiers Media SA}
}

@article{han2019efficient,
  title={Efficient convolutional neural network training with direct feedback alignment},
  author={Han, Donghyeon and Yoo, Hoi-jun},
  journal={arXiv preprint arXiv:1901.01986},
  year={2019}
}

@article{scellier2023energy,
  title={Energy-based learning algorithms for analog computing: a comparative study},
  author={Scellier, Benjamin and Ernoult, Maxence and Kendall, Jack and Kumar, Suhas},
  journal={Advances in neural information processing systems},
  volume={36},
  pages={52705--52731},
  year={2023}
}

@article{crafton2019direct,
  title={Direct feedback alignment with sparse connections for local learning},
  author={Crafton, Brian and Parihar, Abhinav and Gebhardt, Evan and Raychowdhury, Arijit},
  journal={Frontiers in neuroscience},
  volume={13},
  pages={525},
  year={2019},
  publisher={Frontiers Media SA}
}

@article{salvatori2024stablefastfullyautomatic,
  title={A Stable, Fast, and Fully Automatic Learning Algorithm for Predictive Coding Networks}, 
  author={Tommaso Salvatori and Yuhang Song and Yordan Yordanov and Beren Millidge and Zhenghua Xu and Lei Sha and Cornelius Emde and Rafal Bogacz and Thomas Lukasiewicz},
  journal={arXiv preprint arXiv:2212.00720},
  year={2024}
}

@article{pinchetti2024benchmarking,
  title={Benchmarking Predictive Coding Networks--Made Simple},
  author={Pinchetti, Luca and Qi, Chang and Lokshyn, Oleh and Olivers, Gaspard and Emde, Cornelius and Tang, Mufeng and M'Charrak, Amine and Frieder, Simon and Menzat, Bayar and Bogacz, Rafal and others},
  journal={arXiv preprint arXiv:2407.01163},
  year={2024}
}

@article{goemaere2025error,
  title={Error Optimization: Overcoming Exponential Signal Decay in Deep Predictive Coding Networks},
  author={Goemaere, C{\'e}dric and Oliviers, Gaspard and Bogacz, Rafal and Demeester, Thomas},
  journal={arXiv preprint arXiv:2505.20137},
  year={2025}
}

@article{akrout2019deep,
  title={Deep learning without weight transport},
  author={Akrout, Mohamed and Wilson, Collin and Humphreys, Peter and Lillicrap, Timothy and Tweed, Douglas B},
  journal={Advances in neural information processing systems},
  volume={32},
  year={2019}
}

@inproceedings{kolen1994backpropagation,
  title={Backpropagation without weight transport},
  author={Kolen, John F and Pollack, Jordan B},
  booktitle={Proceedings of 1994 IEEE International Conference on Neural Networks (ICNN'94)},
  volume={3},
  pages={1375--1380},
  year={1994},
  organization={IEEE}
}

@article{lillicrap2014random,
  title={Random feedback weights support learning in deep neural networks},
  author={Lillicrap, Timothy P and Cownden, Daniel and Tweed, Douglas B and Akerman, Colin J},
  journal={arXiv preprint arXiv:1411.0247},
  year={2014}
}

@article{webster2020learning,
  title={Learning the connections in direct feedback alignment},
  author={Webster, Matthew Bailey and Choi, Jonghyun and others},
  year={2020}
}

@article{zahid2023predictive,
  title={Predictive coding as a neuromorphic alternative to backpropagation: a critical evaluation},
  author={Zahid, Umais and Guo, Qinghai and Fountas, Zafeirios},
  journal={Neural Computation},
  volume={35},
  number={12},
  pages={1881--1909},
  year={2023},
  publisher={MIT Press One Rogers Street, Cambridge, MA 02142-1209, USA journals-info~…}
}

@article{whittington2017approximation,
  title={An approximation of the error backpropagation algorithm in a predictive coding network with local hebbian synaptic plasticity},
  author={Whittington, James CR and Bogacz, Rafal},
  journal={Neural computation},
  volume={29},
  number={5},
  pages={1229--1262},
  year={2017},
  publisher={MIT Press One Rogers Street, Cambridge, MA 02142-1209, USA journals-info~…}
}

@article{millidge2022theoretical,
  title={A theoretical framework for inference and learning in predictive coding networks},
  author={Millidge, Beren and Song, Yuhang and Salvatori, Tommaso and Lukasiewicz, Thomas and Bogacz, Rafal},
  journal={arXiv preprint arXiv:2207.12316},
  year={2022}
}

@article{millidge2022backpropagation,
  title={Backpropagation at the infinitesimal inference limit of energy-based models: Unifying predictive coding, equilibrium propagation, and contrastive hebbian learning},
  author={Millidge, Beren and Song, Yuhang and Salvatori, Tommaso and Lukasiewicz, Thomas and Bogacz, Rafal},
  journal={arXiv preprint arXiv:2206.02629},
  year={2022}
}

@article{elias2003predictive,
  title={Predictive coding--I},
  author={Elias, Peter},
  journal={IRE Transactions on Information Theory},
  volume={1},
  number={1},
  pages={16--24},
  year={2003},
  publisher={IEEE}
}

@article{elias1955predictive,
  title={Predictive coding--II},
  author={Elias, Peter},
  journal={IRE Transactions on Information Theory},
  volume={1},
  number={1},
  pages={24--33},
  year={1955},
  publisher={IEEE}
}

@article{frenkel2023bottom,
  title={Bottom-up and top-down approaches for the design of neuromorphic processing systems: Tradeoffs and synergies between natural and artificial intelligence},
  author={Frenkel, Charlotte and Bol, David and Indiveri, Giacomo},
  journal={Proceedings of the IEEE},
  volume={111},
  number={6},
  pages={623--652},
  year={2023},
  publisher={IEEE}
}

@article{nokland2016direct,
  title={Direct feedback alignment provides learning in deep neural networks},
  author={N{\o}kland, Arild},
  journal={Advances in neural information processing systems},
  volume={29},
  year={2016}
}

@article{lillicrap2016random,
  title={Random synaptic feedback weights support error backpropagation for deep learning},
  author={Lillicrap, Timothy P and Cownden, Daniel and Tweed, Douglas B and Akerman, Colin J},
  journal={Nature communications},
  volume={7},
  number={1},
  pages={13276},
  year={2016},
  publisher={Nature Publishing Group UK London}
}

@article{frenkel2021learning,
  title={Learning without feedback: Fixed random learning signals allow for feedforward training of deep neural networks},
  author={Frenkel, Charlotte and Lefebvre, Martin and Bol, David},
  journal={Frontiers in neuroscience},
  volume={15},
  pages={629892},
  year={2021},
  publisher={Frontiers Media SA}
}

@article{lillicrap2019backpropagation,
  title={Backpropagation through time and the brain},
  author={Lillicrap, Timothy P and Santoro, Adam},
  journal={Current opinion in neurobiology},
  volume={55},
  pages={82--89},
  year={2019},
  publisher={Elsevier}
}

@article{whittington2019theories,
  title={Theories of error back-propagation in the brain},
  author={Whittington, James CR and Bogacz, Rafal},
  journal={Trends in cognitive sciences},
  volume={23},
  number={3},
  pages={235--250},
  year={2019},
  publisher={Elsevier}
}

@article{grossberg1987competitive,
  title={Competitive learning: From interactive activation to adaptive resonance},
  author={Grossberg, Stephen},
  journal={Cognitive science},
  volume={11},
  number={1},
  pages={23--63},
  year={1987},
  publisher={Elsevier}
}

@article{ellenberger2024backpropagation,
  title={Backpropagation through space, time, and the brain},
  author={Ellenberger, Benjamin and Haider, Paul and Jordan, Jakob and Max, Kevin and Jaras, Ismael and Kriener, Laura and Benitez, Federico and Petrovici, Mihai A},
  journal={arXiv preprint arXiv:2403.16933},
  year={2024}
}

@article{friston2009predictive,
  title={Predictive coding under the free-energy principle},
  author={Friston, Karl and Kiebel, Stefan},
  journal={Philosophical transactions of the Royal Society B: Biological sciences},
  volume={364},
  number={1521},
  pages={1211--1221},
  year={2009},
  publisher={The Royal Society London}
}

@article{friston2006free,
  title={A free energy principle for the brain},
  author={Friston, Karl and Kilner, James and Harrison, Lee},
  journal={Journal of physiology-Paris},
  volume={100},
  number={1-3},
  pages={70--87},
  year={2006},
  publisher={Elsevier}
}

@article{friston2005theory,
  title={A theory of cortical responses},
  author={Friston, Karl},
  journal={Philosophical transactions of the Royal Society B: Biological sciences},
  volume={360},
  number={1456},
  pages={815--836},
  year={2005},
  publisher={The Royal Society London}
}

@article{rao1999predictive,
  title={Predictive coding in the visual cortex: a functional interpretation of some extra-classical receptive-field effects},
  author={Rao, Rajesh PN and Ballard, Dana H},
  journal={Nature neuroscience},
  volume={2},
  number={1},
  pages={79--87},
  year={1999},
  publisher={Nature Publishing Group}
}

@article{huang2011predictive,
  title={Predictive coding},
  author={Huang, Yanping and Rao, Rajesh PN},
  journal={Wiley Interdisciplinary Reviews: Cognitive Science},
  volume={2},
  number={5},
  pages={580--593},
  year={2011},
  publisher={Wiley Online Library}
}

@article{salvatori2023brain,
  title={Brain-inspired computational intelligence via predictive coding},
  author={Salvatori, Tommaso and Mali, Ankur and Buckley, Christopher L and Lukasiewicz, Thomas and Rao, Rajesh PN and Friston, Karl and Ororbia, Alexander},
  journal={arXiv preprint arXiv:2308.07870},
  volume={13},
  year={2023},
  publisher={no}
}

@article{millidge2022predictivefuture,
  title={Predictive coding: Towards a future of deep learning beyond backpropagation?},
  author={Millidge, Beren and Salvatori, Tommaso and Song, Yuhang and Bogacz, Rafal and Lukasiewicz, Thomas},
  journal={arXiv preprint arXiv:2202.09467},
  year={2022}
}

@article{lecun2002gradient,
  title={Gradient-based learning applied to document recognition},
  author={LeCun, Yann and Bottou, L{\'e}on and Bengio, Yoshua and Haffner, Patrick},
  journal={Proceedings of the IEEE},
  volume={86},
  number={11},
  pages={2278--2324},
  year={2002},
  publisher={Ieee}
}

@article{hochreiter1997long,
  title={Long short-term memory},
  author={Hochreiter, Sepp and Schmidhuber, J{\"u}rgen},
  journal={Neural computation},
  volume={9},
  number={8},
  pages={1735--1780},
  year={1997},
  publisher={MIT press}
}

@article{vaswani2017attention,
  title={Attention is all you need},
  author={Vaswani, Ashish and Shazeer, Noam and Parmar, Niki and Uszkoreit, Jakob and Jones, Llion and Gomez, Aidan N and Kaiser, {\L}ukasz and Polosukhin, Illia},
  journal={Advances in neural information processing systems},
  volume={30},
  year={2017}
}

@article{alom2018history,
  title={The history began from alexnet: A comprehensive survey on deep learning approaches},
  author={Alom, Md Zahangir and Taha, Tarek M and Yakopcic, Christopher and Westberg, Stefan and Sidike, Paheding and Nasrin, Mst Shamima and Van Esesn, Brian C and Awwal, Abdul A S and Asari, Vijayan K},
  journal={arXiv preprint arXiv:1803.01164},
  year={2018}
}

@article{krizhevsky2017imagenet,
  title={ImageNet classification with deep convolutional neural networks},
  author={Krizhevsky, Alex and Sutskever, Ilya and Hinton, Geoffrey E},
  journal={Communications of the ACM},
  volume={60},
  number={6},
  pages={84--90},
  year={2017},
  publisher={AcM New York, NY, USA}
}

@article{kingma2013auto,
  title={Auto-encoding variational bayes},
  author={Kingma, Diederik P and Welling, Max},
  journal={arXiv preprint arXiv:1312.6114},
  year={2013}
}

@article{goodfellow2020generative,
  title={Generative adversarial networks},
  author={Goodfellow, Ian and Pouget-Abadie, Jean and Mirza, Mehdi and Xu, Bing and Warde-Farley, David and Ozair, Sherjil and Courville, Aaron and Bengio, Yoshua},
  journal={Communications of the ACM},
  volume={63},
  number={11},
  pages={139--144},
  year={2020},
  publisher={ACM New York, NY, USA}
}

@inproceedings{parmar2018image,
  title={Image transformer},
  author={Parmar, Niki and Vaswani, Ashish and Uszkoreit, Jakob and Kaiser, Lukasz and Shazeer, Noam and Ku, Alexander and Tran, Dustin},
  booktitle={International conference on machine learning},
  pages={4055--4064},
  year={2018},
  organization={PMLR}
}

@article{beck2024xlstm,
  title={xlstm: Extended long short-term memory},
  author={Beck, Maximilian and P{\"o}ppel, Korbinian and Spanring, Markus and Auer, Andreas and Prudnikova, Oleksandra and Kopp, Michael and Klambauer, G{\"u}nter and Brandstetter, Johannes and Hochreiter, Sepp},
  journal={Advances in Neural Information Processing Systems},
  volume={37},
  pages={107547--107603},
  year={2024}
}
\bibliographystyle{icml2026}

%%%%%%%%%%%%%%%%%%%%%%%%%%%%%%%%%%%%%%%%%%%%%%%%%%%%%%%%%%%%%%%%%%%%%%%%%%%%%%%
%%%%%%%%%%%%%%%%%%%%%%%%%%%%%%%%%%%%%%%%%%%%%%%%%%%%%%%%%%%%%%%%%%%%%%%%%%%%%%%
% APPENDIX
%%%%%%%%%%%%%%%%%%%%%%%%%%%%%%%%%%%%%%%%%%%%%%%%%%%%%%%%%%%%%%%%%%%%%%%%%%%%%%%
%%%%%%%%%%%%%%%%%%%%%%%%%%%%%%%%%%%%%%%%%%%%%%%%%%%%%%%%%%%%%%%%%%%%%%%%%%%%%%%
\newpage
\appendix
\onecolumn
\section{Appendix}

\subsection{Convergence of Feedback Matrices under the Direct Kolen–Pollack Algorithm}  
\label{app:dkp}
In this appendix, we extend the empirical observations of \citet{webster2020learning} by providing a mathematical argument for why DKP achieves a better alignment with BP than DFA. While their work demonstrates this empirically, it does not provide a formal theoretical justification, instead attributing the behavior to analogies with KP. However, while it has been proven for KP that the feedback matrices $\Psi_\ell$ converge to $\Theta_\ell^\top$ for all layers \citep{kolen1994backpropagation, akrout2019deep}, in DKP this result strictly holds only for the last layer, as all other hidden layers have a dimensionality mismatch between $\Psi_\ell$ and $\Theta_\ell$. Here, we offer a novel theoretical perspective on the work of \citet{webster2020learning}, demonstrating that DKP drives the feedback matrices toward values approximating the chain of transposed forward weights, similar to BP, up to the Moore-Penrose pseudoinverse.

Let us consider a feedforward neural network with $L$ layers, where $d_{\ell}$ denotes the number of neurons at layer $\ell$. The weight matrix $\Theta_\ell \in \mathbb{R}^{d_{\ell+1}\times d_{\ell}}$ maps activations from layer $x_\ell \in \mathbb{R}^{d_\ell}$ to the next layer $x_{\ell+1} \in \mathbb{R}^{d_{\ell+1}}$, according to  
\begin{equation}
x_{\ell+1} = f\!\left(\Theta_\ell x_\ell\right),
\label{eq:forward}
\end{equation}  
where $f(\cdot)$ denotes an arbitrary non-linear activation function. In the following derivations, we assume an identity activation function, so that the derivative of the activation function can be omitted.  

In DFA, error feedback is provided by fixed random matrices $\Psi_\ell \in \mathbb{R}^{d_\ell \times d_L}$ that project the output error term $\delta_L$ directly to each hidden layer $\ell$. These matrices replace the layer-specific error term $\delta_\ell$ used in BP, computed by transporting the error backward from the next layer through the transposed weight matrix $\Theta_{\ell-1}^{\top}$, as described by  
\begin{equation}
\delta_\ell = \Theta_{\ell+1}^\top \delta_{\ell+1}.
\label{eq:bp-error}
\end{equation}  
In DFA, as illustrated in Figure~\ref{fig:circuits}, this layer-specific error term $\delta_\ell$ is instead approximated by  
\begin{equation}
\tilde{\delta_\ell} = \Psi_\ell \delta_L.
\label{eq:dfa-error}
\end{equation}

In contrast, DKP allows the feedback matrices to be updated following the local update rule
\begin{equation}
\Delta \Psi_\ell = x_{\ell}\,\delta_L^\top.
\label{eq:dkp-update}
\end{equation} 
With both forward and feedback weights subject to a decay term, this learning has been empirically shown to enable $\Psi_\ell$ to provide a more BP-aligned update for $\Theta_\ell$ compared to standard DFA \citep{webster2020learning}. We extend their work by demonstrating that the feedback matrices gradually align with the pseudoinverses of the forward weight matrices in a recursive dependency, thereby yielding a closer approximation of BP error propagation compared to DFA.  

Starting from the last layer, the update rule for the weight matrix preceding it, denoted as $\Theta_{L-1} \in \mathbb{R}^{d_L \times d_{L-1}}$, is the same for BP, DFA, and DKP, and is given by
\begin{equation}
\Delta \Theta_{L-1} = -\alpha \left( \delta_L x_{L-1}^\top + \Theta_{L-1} \right),
\label{eq:last-layer-update}
\end{equation}
where $\alpha \in \left(0,1\right)$ is the learning rate, and the second term of the update is the weight decay term. According to Eq.~\eqref{eq:dkp-update} and assuming the same learning rate for the feedback weights, the update rule for the feedback matrix connecting the last layer to the penultimate one is given by
\begin{equation}
\Delta \Psi_{L-1} = -\alpha \left( x_{L-1} \delta_L^\top + \Psi_{L-1} \right).
\label{eq:last-layer-update-dkp}
\end{equation}

In this specific case, $\Theta_{L-1}$ has the shape of $\Psi_{L-1}^{\top}$ and their updates are transposes of each other. As training progresses, both matrices converge to the same value, since the contribution of the initial condition vanishes under the effect of the learning rate $\alpha$ \citep{kolen1994backpropagation, akrout2019deep}.
Indeed, following KP, by defining  
\begin{equation}
\Omega_{L-1}(t+1) = \Theta_{L-1}(t+1) - \Psi_{L-1}^\top(t+1),
\label{eq:def-omega}
\end{equation}
and using the update rules Eq.~\eqref{eq:last-layer-update} and Eq.~\eqref{eq:last-layer-update-dkp}, we obtain
\begin{equation}
\begin{aligned}
\Omega_{L-1}(t+1) &= \left( \Theta_{L-1}(t) + \Delta \Theta_{L-1}(t) \right) 
- \left( \Psi_{L-1}^\top(t) + \Delta \Psi_{L-1}^\top(t) \right) \\
&= \Theta_{L-1}(t) - \Psi_{L-1}^\top(t) - \alpha \left( \delta_L(t) x_{L-1}^\top(t) + \Theta_{L-1}(t) - \delta_L(t) x_{L-1}^\top(t) - \Psi_{L-1}^\top(t) \right) \\
&= \Theta_{L-1}(t) - \Psi_{L-1}^\top(t) - \alpha \left(  \Theta_{L-1}(t) -  \Psi_{L-1}^\top(t) \right) \\
&= \big(1-\alpha\big) \left(  \Theta_{L-1}(t) -  \Psi_{L-1}^\top(t) \right) \\
&= \big(1-\alpha\big) \Omega_{L-1}(t),
\end{aligned}
\label{eq:kp-last-layer}
\end{equation}
We can now further unroll Eq.~\eqref{eq:kp-last-layer} in time, resulting in
\begin{equation}
\begin{aligned}
\Omega_{L-1}(t+1) &= \big(1-\alpha\big)^t\Omega_{L-1}(0)\\
 &= \big(1-\alpha\big)^t \left( \Theta_{L-1}(0) - \Psi_{L-1}^\top(0) \right).
\end{aligned}
\label{eq:kp-final}
\end{equation}
Therefore, $\Omega_{L-1}(t)$ converges to zero as training progresses, with the initial difference between the forward and feedback matrices decaying exponentially due to the learning rate $\alpha$. In other words, we have that 
\begin{equation}
\lim_{t \to \infty} \Omega_{L-1}(t) = 0 
\quad \implies \quad 
\lim_{t \to \infty} \Psi_{L-1}^\top(t) = \Theta_{L-1}(t).
\label{eq:convergence-pk}
\end{equation}
The DKP update rule for $\Theta_{L-2}$ is given by
\begin{equation}
\begin{aligned}
    \Delta \Theta_{L-2} &= -\alpha \left( \tilde{\delta}_{L-1} x_{L-2}^\top + \Theta_{L-2}\right) \\
    &= -\alpha \left( \Psi_{L-1}\delta_{L} x_{L-2}^\top + \Theta_{L-2} \right),\\
\end{aligned}
\label{eq:thetaL-2-final}
\end{equation}
which effectively approximates the BP update for $\Theta_{L-2}$ and ultimately matches it as $t \rightarrow \infty$, as using Eq.~\eqref{eq:convergence-pk} yields
\begin{equation}
\begin{aligned}
    \lim_{t \to \infty} \Delta \Theta_{L-2} = -\alpha \left( \Theta_{L-1}^\top\delta_{L} x_{L-2}^\top + \Theta_{L-2} \right).
\end{aligned}
\label{eq:convergence_thetaL-2}
\end{equation}
Here and throughout the rest of this appendix, we omit the time index since we consider the limit \(t \to \infty\). Therefore, we approximate \(\Psi_{L-1}\) by \(\Theta_{L-1}^\top\), noting that this approximation introduces a small error, since exact equality holds only at the limit.  

Unfortunately, the convergence obtained for $\Psi_{L-1}$ in Eq.~\eqref{eq:convergence-pk} does not directly apply to \(\Psi_{L-2} \in \mathbb{R}^{d_{L-2} \times d_L}\), as its dimensions do not match those of \(\Theta_{L-2} \in \mathbb{R}^{d_{L-1} \times d_{L-2}}\). The update rule for \(\Psi_{L-2}\), given by
\begin{equation}
    \Delta \Psi_{L-2} = -\alpha \big( x_{L-2} \delta_L^\top + \Psi_{L-2} \big),
\label{eq:psi2}
\end{equation}
can be substituted into Eq.~\eqref{eq:convergence_thetaL-2}, leading to
\begin{equation}
\begin{aligned}
    \Delta \Theta_{L-2} &= -\alpha \left( \Theta_{L-1}^\top\delta_{L} x_{L-2}^\top + \Theta_{L-2} \right)\\
    &= -\alpha \left( \Theta_{L-1}^\top \left( -\alpha^{-1} \Delta \Psi_{L-2}^\top - \Psi_{L-2}^\top \right) + \Theta_{L-2} \right) \\
    &= \Theta_{L-1}^\top \Delta \Psi_{L-2}^\top -\alpha \left(\Theta_{L-2} - \Theta_{L-1}^\top \Psi_{L-2}^\top \right).
\end{aligned}
\label{eq:psiL-2-final-p1}
\end{equation}
Here, we neglect the decay terms, since they vanish asymptotically during training, as previously discussed for $\Theta_{L-1}$ and $\Psi_{L-1}$. We can now transpose both sides, and multiply them by $\Theta_{L-1}^\top$, resulting in
\begin{equation}
    \Delta \Psi_{L-2}\Theta_{L-1}\Theta_{L-1}^\top = \Delta \Theta_{L-2}^\top\Theta_{L-1}^\top.
\label{eq:psiL-2-final-transposed}
\end{equation}
This allows us to link the update of $\Psi_{L-2}$ to that of $\Theta_{L-2}$ through
\begin{equation}
\begin{aligned}
    \Delta \Psi_{L-2} &= \Delta \Theta_{L-2}^\top\Theta_{L-1}^\top \left( \Theta_{L-1}\Theta_{L-1}^\top \right)^{-1}\\
    &= \Delta \Theta_{L-2}^\top \Theta_{L-1}^{+},
\end{aligned}
\label{eq:psiL-2-final-final}
\end{equation}
where $\Theta_{L-1}^{+} \in \mathbb{R}^{d_{L-1} \times d_{L}}$ is the Moore-Penrose pseudoinverse of $\Theta_{L-1}$, assuming the latter has full row rank. On the one hand, the BP update of $\Theta_{L-3}$ is given by
\begin{equation}
\begin{aligned}
    \Delta \Theta_{L-3} &= -\alpha \big( \delta_{L-2} x_{L-3}^\top + \Theta_{L-3}\big) \\
    &= -\alpha \big( \Theta_{L-2}^\top\delta_{L-1} x_{L-3}^\top + \Theta_{L-3}\big) \\
    &= -\alpha \big( \Theta_{L-2}^\top\Theta_{L-1}^\top\delta_{L} x_{L-3}^\top + \Theta_{L-3}\big).
\end{aligned}
\end{equation}
On the other hand, by making use of Eq.~\eqref{eq:psiL-2-final-final}, which is valid under the previously mentioned assumptions and approximations, the DKP update can be expressed as
\begin{equation}
\begin{aligned}
    \Delta \Theta_{L-3} &= -\alpha \big( \tilde{\delta}_{L-2} x_{L-3}^\top + \Theta_{L-3}\big) \\
    &= -\alpha \big( \Psi_{L-2}^\top\delta_{L} x_{L-3}^\top + \Theta_{L-3}\big) \\
    &= -\alpha \big( \Theta_{L-2}^\top \Theta_{L-1}^{+} \delta_{L} x_{L-3}^\top + \Theta_{L-3}\big).
\end{aligned}
\label{eq:theta3-dkp}
\end{equation}
Hence, the DKP update of $\Theta_{L-3}$ approximates the BP one, and is exact if $\Theta_{L-1}^+ = \Theta_{L-1}^{\top}$, meaning that $\Theta_{L-1}$ is orthogonal. We now move on to the DKP update of $\Psi_{L-3}$, given by
\begin{equation}
\Delta \Psi_{L-3} = -\alpha \big( x_{L-3} \delta_{L}^\top + \Psi_{L-3}\big).
\label{eq:psi-3}
\end{equation}
By repeating the same procedure as for $\Psi_{L-2}$, we substitute Eq.~\eqref{eq:psi-3} into Eq.~\eqref{eq:theta3-dkp}, resulting in
\begin{equation}
\begin{aligned}
\Delta \Theta_{L-3} &= -\alpha \big( \Theta_{L-2}^\top \Theta_{L-1}^{+} \big( -\alpha^{-1}\Delta \Psi_{L-3} -  \Psi_{L-3} \big)^\top + \Theta_{L-3}\big)\\
&= \Theta_{L-2}^\top \Theta_{L-1}^{+} \Delta \Psi_{L-3}^\top - \alpha \big(\Theta_{L-3} - \Theta_{L-2}^\top \Theta_{L-1}^{+} \Psi_{L-3}^\top \big).
\end{aligned}
\label{eq:psi-3intheta3}
\end{equation}
We now again do not consider the terms decaying with time, as they tend to zero as the training goes on, and focus only on the term the update converges to. As done previously, after dropping the decay term, we transpose both sides, yielding
\begin{equation}
 \Delta\Psi_{L-3} \left( \big(\Theta_{L-1}^{+} \big)^{\top}\Theta_{L-2} \right) = \Delta \Theta_{L-3}^\top,
\label{eq:psi-3intheta3-1}
\end{equation}
and multiply them by $\left( \big(\Theta_{L-1}^{+} \big)^{\top}\Theta_{L-2} \right)^{\top}$, leading to
\begin{equation}
 \Delta\Psi_{L-3} \left( \big(\Theta_{L-1}^{+} \big)^{\top}\Theta_{L-2} \right) \left( \big(\Theta_{L-1}^{+} \big)^{\top}\Theta_{L-2} \right)^{\top} = \Delta \Theta_{L-3}^\top \left( \big(\Theta_{L-1}^{+} \big)^{\top}\Theta_{L-2} \right)^{\top}.
\label{eq:psi-3intheta3-2}
\end{equation}
Lastly, by multiplying both sides by the inverse of the product of matrices on the right of $\Delta\Psi_{L-3}$, we obtain
\begin{equation}
\begin{aligned}
 \Delta\Psi_{L-3} &= \Delta \Theta_{L-3}^\top \left( \big(\Theta_{L-1}^{+} \big)^{\top}\Theta_{L-2} \right)^{\top} \left[\left( \big(\Theta_{L-1}^{+} \big)^{\top}\Theta_{L-2} \right) \left( \big(\Theta_{L-1}^{+} \big)^{\top}\Theta_{L-2} \right)^{\top}\right]^{-1} \\
 &= \Delta \Theta_{L-3}^\top \left( \big(\Theta_{L-1}^{+} \big)^{\top}\Theta_{L-2} \right)^{+},\\
 &= \Delta \Theta_{L-3}^\top \left( \big(\Theta_{L-1}^{\top} \big)^{+}\Theta_{L-2} \right)^{+},
 \end{aligned}
\label{eq:psi-3intheta3-3}
\end{equation}
where again, the feedback matrix includes a chain of forward matrix pseudoinverses. More generally, under the assumption of $t \rightarrow \infty$ and that $\Theta_\ell$ is a full row rank rectangular matrix, we have
\begin{equation}
\Psi_{L-\ell} =
\begin{cases}
    \Theta_{L-\ell}^\top & \text{if } \ell = 1,\\
    \Theta_{L-\ell}^\top \left( \Psi_{L-\ell+1}^\top \right)^{+} & \text{if } 1 < \ell < L,
\end{cases}
\label{eq:feedback-piecewise}
\end{equation}

It should be noted that in practice, the assumptions we made are never perfectly met, since neural network training does not proceed for an infinite number of iterations, meaning that the decay terms we neglected do not completely vanish, and also involves the derivatives of the non-linear activation functions. Nonetheless, our derivation demonstrates how DKP extends DFA by updating the feedback random matrices with terms that also appear in BP's error propagation, providing a clearer understanding of why DKP achieves better alignment and consequently improves performance compared to standard DFA \cite{webster2020learning}.

\subsection{Error propagation delay}
\label{app:error_delay}

In this section, we introduce and provide a formal proof of Theorem~\ref{th:error_delay_theorem}, which quantifies the delay in error propagation in forward-initialized PC networks. Specifically, we show that the feedback error signal reaches a given layer with a delay equal to its distance from the output \citep{zahid2023predictive}.

\begin{theorem}[Error propagation delay]  
Consider a forward-initialized PC network with discrete-time updates. Assuming an incorrect prediction, the neural activity $\phi_\ell$ at layer $\ell$ requires at least $\hat{t} = L - \ell$ inference-phase steps before it deviates from equilibrium and begins to evolve according to 
Eq.~\eqref{eq:neural_act_learn_rule} (proof provided in Appendix~\ref{app:error_delay}).
\label{th:error_delay_theorem}
\end{theorem}

\begin{proof}
Let us consider a forward-initialized PC network with $L+1$ layers, where the neural activity evolves in discrete time steps $t \in \mathbb{N}_0$, and each layer is initialized as $\phi_\ell(0) = f(\Theta_{\ell-1}\phi_{\ell-1}(0))$. The error neurons at layer $\ell$ are defined as $\epsilon_\ell(0) = \phi_\ell(0) - f(\Theta_{\ell-1}\phi_{\ell-1}(0))$. By construction, after forward initialization, all error neurons vanish and the network's dynamics is at equilibrium:
\begin{equation}
\begin{aligned}
    \|\epsilon_\ell(0)\|_2^2 &= \|\phi_\ell(0) - f(\Theta_{\ell-1}\phi_{\ell-1}(0))\|_2^2 \\
    &= \|\phi_\ell(0) - \phi_\ell(0)\|_2^2 \\
    &= 0 \quad \text{for} \quad 0 < \ell \leq L.
\end{aligned}
\end{equation}

At the beginning of the inference phase, where the network's neural activity evolves to minimize Eq.~\eqref{eq:fe_sum_errors}, we clamp the target vector $y$ to the output layer $\phi_L$. Assuming an incorrect prediction, i.e., $y - f(\Theta_{L-1}\phi_{L-1}(0)) \neq 0$, the prediction error satisfies
\begin{equation}
\|\epsilon_L(0)\|_2^2 = \|y - f(\Theta_{L-1}\phi_{L-1}(0))\|_2^2 \geq 0.
\end{equation}

Consequently, according to Eq.~\eqref{eq:neural_act_learn_rule}, the activity at layer $L-1$ has a non-zero update at $t=1$:
\begin{equation}
\begin{aligned}
    \phi_{L-1}(1) &= \phi_{L-1}(0) + \Delta \phi_\ell(0)\\
    &= \phi_{L-1}(0) - \gamma \frac{\partial F}{\partial \phi_{L-1}}(0)\\ 
    &= \phi_{L-1}(0) - \gamma \Big( J_{\phi_{L-1}}(0)^\top  \epsilon_L - \epsilon_{L-1} \Big) \\
    &= \phi_{L-1}(0) - \gamma \Big( J_{\phi_{L-1}}(0)^\top  \epsilon_L \Big),
\end{aligned}
\end{equation}
where $J_{\phi_{\ell}}(0) = \frac{\partial f(\Theta_{\ell} \phi_{\ell})}{\partial \phi_{\ell}}(0) \in \mathbb{R}^{d_{\ell+1} \times d_\ell }$ is the Jacobian matrix of prediction by layer $\ell$ with respect to its neural activity. For all preceding layers, we have that $\Delta \phi_\ell(0) = 0$, since both $\epsilon_\ell$ and $\epsilon_{\ell+1}$ are null.

After $\phi_{L-1}$ has been updated at time $t=1$, the corresponding error becomes non-zero:
\begin{equation}
\begin{aligned}
    \|\epsilon_{L-1}(1)\|_2^2 &= \|\phi_{L-1}(1) - f(\Theta_{L-2}, \phi_{L-2}(1))\|_2^2\\
    &= \|\phi_{L-1}(1) - f(\Theta_{L-2}, \phi_{L-2}(0))\|_2^2\\
    &= \|\phi_{L-1}(1) - \phi_{L-1}(0)\|_2^2\\
    &= \|\Delta \phi_{L-1}(1)\|_2^2\\
    &\geq 0.
\end{aligned}
\end{equation}

At the subsequent timestamp $t=2$, layer $L-2$ also receives a non-zero error and updates accordingly:
\begin{equation}
\begin{aligned}
    \phi_{L-2}(2) &= \phi_{L-2}(1)+\Delta \phi_{L-2}(1)\\
    &= \phi_{L-2}(1) - \gamma \frac{\partial F}{\partial \phi_{L-2}}(1)\\ 
    &= \phi_{L-2}(0) - \gamma \frac{\partial F}{\partial \phi_{L-2}}(0)\\ 
    &= \phi_{L-2}(0) - \gamma \Big( J_{\phi_{L-2}}(0)^\top \epsilon_{L-1} - \epsilon_{L-2} \Big) \\
    &= \phi_{L-2}(0) - \gamma \Big( J_{\phi_{L-2}}(0)^\top \epsilon_{L-1} \Big),
\end{aligned}
\end{equation}
where $\phi_{L-2}(1) = \phi_{L-2}(0)$ and $\frac{\partial F}{\partial \phi_{L-2}}(1) = \frac{\partial F}{\partial \phi_{L-2}}(0)$, since $\phi_{L-2}$ has remained unchanged at $t=1$, due to $\epsilon_{L-2}(0) = \epsilon_{L-1}(0) = 0$.
Again, at time $t=2$ all previous layers $\phi_\ell(2) \, , \,0 < \ell < L-2$ remain unchanged, as $\epsilon_{L-\ell}(1) = \epsilon_{L-\ell+1}(1) = 0$. By induction, it follows that under these assumptions, any layer $\ell$ requires at least $\hat{t} = L - \ell$ timestamps to update its neural activity and corresponding error neurons.  

For the general case, at a specific time $t$, the error neurons at layer $L-t$ can be expressed as
\begin{equation}
\begin{aligned}
    \|\epsilon_{L-t}(t)\|_2^2 &= \|\phi_{L-t}(t) - f(\Theta_{L-t-1}, \phi_{L-t-1}(t-1))\|_2^2\\
    &= \|\phi_{L-t}(t) - f(\Theta_{L-t-1}, \phi_{L-t-1}(0))\|_2^2\\
    &= \|\phi_{L-t}(t) - \phi_{L-t}(0)\|_2^2,
\end{aligned}
\end{equation}

where $\phi_{L-t}(t) - \phi_{L-t}(0) \neq 0$ if and only if $\phi_{L-t}(t) \neq \phi_{L-t}(0)$. Crucially, this condition can occur only after $\hat{t} = L - \ell$ timestamps. According to Eq.~\eqref{eq:neural_act_learn_rule}, the update of layer $L-t$ depends on the errors from the current and subsequent layers, $\epsilon_{L-t}$ and $\epsilon_{L-t+1}$, respectively. Since $\epsilon_{L-t}$ is itself blocked until $\phi_{L-t}$ changes, the driving term is provided by $\epsilon_{L-t+1}$. However, the latter becomes non-zero only after the previous layer has been updated. Therefore, the propagation of activity changes and error signals strictly follows the network's hierarchy, advancing at most one layer per timestamp, starting from $\epsilon_L$ at $t=0$ and reaching layer $\ell$ only after $\hat{t} = L - \ell$ steps.
\renewcommand{\qedsymbol}{}
\end{proof}

\subsection{Error exponential decay}
\label{app:error_decay}

In this section, we introduce and provide a formal proof of  Theorem~\ref{th:error_decay_theorem}, showing that the squared-$\ell_2$-norm of the feedback error signal for a layer $\ell$ generated at time $\hat{t} = L - \ell$ is bounded by a quantity that decays exponentially proportional to $t$ \citep{goemaere2025error}.

\begin{theorem}[Error exponential decay]
Consider a forward-initialized PC network with discrete-time updates. Assuming an incorrect prediction, the squared $\ell_2$-norm of the feedback error signal at layer $\ell$, $\|\epsilon_\ell(\hat{t})\|_2^2$, at time $\hat{t} = L - \ell$, is upper-bounded by a quantity that decays $\propto \gamma^{2(L-\ell)}$ (proof provided in Appendix~\ref{app:error_decay}).
\label{th:error_decay_theorem}
\end{theorem}

\begin{proof}
Let us consider a forward-initialized PC network with $L+1$ layers, where the neural activity evolves in discrete time steps $t \in \mathbb{N}_0$, and each hidden layer $\phi_\ell(t)$ is initialized as $\phi_\ell(0) = f(\Theta_{\ell-1}\phi_{\ell-1}(0))$. The error neurons are defined as $\epsilon_\ell(t) = \phi_\ell(t) - f(\Theta_{\ell-1}\phi_{\ell-1}(t))$. Here, we make the time dependence of neural activities explicit, as they evolve with time during the inference phase of PC. Considering the update of an arbitrary neuron $\phi_{\ell-1} \, , \, 1 < \ell \leq L$ at time $\hat{t} + 1$, with $\hat{t} = L - \ell$:
\begin{equation}
        \phi_{\ell-1}(\hat{t}+1) = \phi_{\ell-1}(\hat{t}) - \gamma \frac{\partial F}{\partial \phi_{\ell-1}}(\hat{t}), 
\label{eq:update_neural_act_theorem}
\end{equation}
where we can move the on the left side the current neural activity value, obtaining:
\begin{equation}
\begin{aligned}
        \phi_{\ell-1}(\hat{t}+1) - \phi_{\ell-1}(\hat{t}) &= - \gamma \frac{\partial F}{\partial \phi_{\ell-1}}(\hat{t})\\
        \Delta \phi_{\ell-1}(\hat{t}+1)&=- \gamma \frac{\partial F}{\partial \phi_{\ell-1}}(\hat{t})\\
        &= - \gamma \left( \epsilon_{\ell-1}(\hat{t}) - \frac{\partial f\left(\Theta_{\ell-1} \phi_{\ell-1}(\hat{t})\right)}{\partial \phi_{\ell-1} (\hat{t})}^\top \epsilon_{\ell}(\hat{t}) \right)\\
        &= \gamma \left( 
        \frac{\partial f\left(\Theta_{\ell-1} \phi_{\ell-1}(\hat{t})\right)}{\partial \phi_{\ell-1} (\hat{t})}^\top \epsilon_{\ell}(\hat{t})
        \right),
\end{aligned}
\label{eq:unroll_diff}
\end{equation}

where $\epsilon_{\ell-1}(\hat{t}) = 0$ is implied by Theorem~\ref{th:error_delay_theorem}, since $L - \ell-1 < \hat{t}$. We defined $\frac{\partial f\left(\Theta_{\ell-1} \phi_{\ell-1}(\hat{t})\right)}{\partial \phi_{\ell-1} (\hat{t})} = J_{\phi_{\ell-1}} \in \mathbb{R}^{d_{\ell} \times d_{\ell-1}}$ the Jacobian matrix of $f$ at layer $\ell-1$. Here, we have omitted time in $\phi_{\ell-1}$ for brevity, as $\phi_{\ell-1}(\hat{t}) = \phi_{\ell-1}(0)$, following again from Theorem~\ref{th:error_delay_theorem}. Continuing from Eq.~\eqref{eq:unroll_diff}, we expand the error neurons according to their definition and apply Theorem~\ref{th:error_delay_theorem}, resulting in:

\begin{equation}
\begin{aligned}
        \Delta \phi_{\ell-1}(\hat{t}+1)&= \gamma J_{\phi_{\ell-1}}^\top \left[ \phi_{\ell}(\hat{t}) - f\left(\Theta_{\ell-1} \phi_{\ell-1}(\hat{t})\right) \right] \\
        &= \gamma J_{\phi_{\ell-1}}^\top \left[ \phi_{\ell}(\hat{t}-1) - \gamma \frac{\partial F}{\partial \phi_{\ell}}(\hat{t}-1) - f\left(\Theta_{\ell-1} \phi_{\ell-1}(0)\right) \right]\\
        &= \gamma J_{\phi_{\ell-1}}^\top \left[ \phi_{\ell}(0) - \gamma \frac{\partial F}{\partial \phi_{\ell}}(\hat{t}-1) - f(\Theta_{\ell-1} \phi_{\ell-1}(0)) \right] \\
        &= \gamma J_{\phi_{\ell-1}}^\top \left[ f\left(\Theta_{\ell-1} \phi_{\ell-1}(0)\right) - \gamma \frac{\partial F}{\partial \phi_{\ell}}(\hat{t}-1) - f\left(\Theta_{\ell-1} \phi_{\ell-1}(0)\right) \right],
\end{aligned}
\end{equation}
where on the first line we have substituted $\phi_{\ell}(\hat{t})$ with its definition in Eq.~\eqref{eq:update_neural_act_theorem}. According to Theorem~\ref{th:error_delay_theorem}, $\phi_{\ell}(\hat{t}-1) = \phi_{\ell}(0)$ as $L -\ell < \hat{t}-1$, thus this allows to re-write it using the forward initialization definition, as done on the fourth line. By repeating the same steps as in Eq.~\eqref{eq:unroll_diff}, we obtain

\begin{equation}
\begin{aligned}
        \Delta \phi_{\ell-1}(\hat{t}+1)&= \gamma J_{\phi_{\ell-1}}^\top \left[  - \gamma \frac{\partial F}{\partial \phi_{\ell}}(\hat{t}-1) \right]\\
        &= \gamma J_{\phi_{\ell-1}}^\top \left[ - \gamma \left( \epsilon_{\ell}(\hat{t}-1) - J_{\phi_{\ell}}^\top \epsilon_{\ell+1}(\hat{t}-1) \right) \right]\\
        &= \gamma^2 J_{\phi_{\ell-1}}^\top J_{\phi_{\ell}}^\top \epsilon_{\ell+1}(\hat{t}-1).
\end{aligned}
\end{equation}

 By unrolling the formulation backward until the last layer, we get:
\begin{equation}
    \Delta \phi_{\ell-1}(\hat{t}+1) = \gamma^{L-\ell+1}  \left( \prod_{i=0}^{L-\ell} J_{\phi_{\ell-1+i}}^\top \right) \epsilon_L(0).
\label{eq:delta_unroll}
\end{equation}
By substituting Eq.~\eqref{eq:unroll_diff} in Eq.~\eqref{eq:delta_unroll}, we can continue as follows:
\begin{equation}
    \gamma \left( J_{\phi_{\ell-1}}^\top \epsilon_{\ell}(\hat{t}) \right) = \gamma^{L-\ell+1}  \left( \prod_{i=0}^{L-\ell} J_{\phi_{\ell-1+i}}^\top \right) \epsilon_L(0),
\end{equation}
which can be further simplified by eliding common terms, resulting in: 
\begin{equation}
    \epsilon_{\ell}(\hat{t}) = \gamma^{L-\ell}  \left( \prod_{i=1}^{L-\ell} J_{\phi_{\ell-1+i}}^\top \right) \epsilon_L(0).
\label{eq:subbound}
\end{equation}

By writing the squared-$\ell_2$ norm of  Eq.~\eqref{eq:subbound}, we finally derive the upper-bound for the error term:

\begin{equation}
\begin{aligned}
    \|\epsilon_{\ell}(\hat{t})\|_2^2 &= \| \gamma^{L-\ell}  \left( \prod_{i=1}^{L-\ell} J_{\phi_{\ell-1+i}}^\top \right) \epsilon_L(0)\|_2^2\\
    &\leq \gamma^{2(L-\ell)}  \left( \prod_{i=1}^{L-\ell} \|J_{\phi_{\ell-1+i}}^\top\|_2^2 \right) \|\epsilon_L(0)\|_2^2.
\end{aligned}
\end{equation}

\renewcommand{\qedsymbol}{}
\end{proof}

\subsection{Theoretical and Empirical Analysis of DKP-PC Integration}
\label{app:dkp-pc}
In this appendix, we further analyse DKP-PC, providing the reader with additional theoretical and empirical insights into the mechanisms through which the DKP and PC stages interact. Building on the results presented in the main text, we show how the single-step neural activity optimization introduced by the PC stage improves both the updates of the forward and feedback weights of the network.

\subsubsection{Theoretical Analysis of DKP-PC Updates}
\label{app:dkp-pc-theoretical}
In this subsection, we derive analytically the neural activity and feedback weight updates within the proposed DKP-PC algorithm. We will follow the steps of the DKP-PC algorithm outlined in Algorithm~\ref{alg:dkppc-alg}. Besides, for the sake of mathematical tractability, we make the same assumption of linearity of the network as in Appendix~\ref{app:dkp}.

\textit{1) Direct feedback alignment update} -- First, the forward weights $\Theta_\ell$ are updated using the approximate gradients $\tilde{\delta}_\ell = \Psi_\ell \delta_L$, provided through the DKP feedback matrices $\Psi_\ell$. This results in updated forward weights
\begin{equation}
    \tilde{\Theta}_\ell = \Theta_\ell + \Delta \Theta_\ell\textrm{,}
\end{equation}
where the DKP forward weights' update is given by
\begin{equation}
    \Delta \Theta_\ell = -\alpha \tilde{\delta}_{\ell+1} \phi_\ell^\top = -\alpha \Psi_{\ell+1} \delta_L \phi_\ell^\top\label{eq:dkp_theta_update}\textrm{.}
\end{equation}

\textit{2) Inference phase} -- The DKP forward weights' update is followed by a single update of the neural activity $\phi_\ell$, aiming to minimize the network's FE, as opposed to several steps for usual PC algorithms for spreading the error information. This leads to the following optimized neural activity:
\begin{equation}
    \phi_\ell^{*} = \phi_\ell + \Delta \tilde{\phi}_{\ell}\textrm{.}
\label{eq:new_na_after_dkp}
\end{equation}
The updated neural activity $\phi_\ell^{*}$ now incorporates the information injected in the forward weights through the DKP update, as the error neurons are computed with the new forward weights values as $\tilde{\epsilon}_{\ell+1} = \phi_{\ell+1} - \tilde{\Theta}_\ell \phi_\ell$, as shown by computing the neural activity update $\Delta \tilde{\phi}_{\ell}$ based on the FE after the DKP weight update $\tilde{F}$:
\begin{equation}
\begin{aligned}
    \Delta \tilde{\phi}_{\ell} &= -\gamma \frac{\partial \tilde{F}}{\partial \phi_\ell}\\
    &= \gamma \left( \tilde{\Theta}_\ell^\top \tilde{\epsilon}_{\ell+1} - \tilde{\epsilon}_\ell \right)\\
    &=  \gamma \left[ \tilde{\Theta}_\ell^\top \left( \phi_{\ell+1} - \tilde{\Theta}_\ell \phi_\ell \right) - \left( \phi_{\ell} - \tilde{\Theta}_{\ell-1} \phi_{\ell-1} \right) \right] \\
    &= \gamma \left[ \left(\Theta_\ell + \Delta \Theta_\ell\right)^\top \left( \phi_{\ell+1} - \left(\Theta_\ell + \Delta \Theta_\ell\right) \phi_\ell \right) - \left( \phi_{\ell} - \left(\Theta_{\ell-1} + \Delta \Theta_{\ell-1}\right) \phi_{\ell-1} \right) \right]  \\
    &= \gamma \left[ \left(\Theta_\ell + \Delta \Theta_\ell\right)^\top \left( \phi_{\ell+1} - \Theta_\ell \phi_\ell - \Delta \Theta_\ell \phi_\ell \right) - \left( \phi_{\ell} - \Theta_{\ell-1}\phi_{\ell-1} - \Delta \Theta_{\ell-1}\phi_{\ell-1} \right) \right]  \\
    &= \gamma \left[ \left(\Theta_\ell + \Delta \Theta_\ell\right)^\top \left( \epsilon_{\ell+1} - \Delta \Theta_\ell \phi_\ell \right) - \left( \epsilon_{\ell} - \Delta \Theta_{\ell-1}\phi_{\ell-1} \right) \right]  \\
    %&\propto  \Theta_\ell^\top\epsilon_{\ell+1} - \Theta_\ell^\top \Delta \Theta_\ell \phi_\ell + \Delta \Theta_\ell^\top \epsilon_{\ell+1} - \Delta \Theta_\ell^\top \Delta \Theta_\ell \phi_\ell - \epsilon_{\ell} + \Delta \Theta_{\ell-1}\phi_{\ell-1} \\
    &= \gamma \left[ \Theta_\ell^\top \epsilon_{\ell+1} - \epsilon_{\ell} + \Delta \Theta_\ell^\top \epsilon_{\ell+1} - \left(\Theta_\ell + \Delta\Theta_\ell\right)^\top \Delta \Theta_\ell \phi_\ell + \Delta \Theta_{\ell-1}\phi_{\ell-1} \right] \\
    &= \gamma \left( {\nabla_{\phi_\ell} F } + \Delta \Theta_\ell^\top \epsilon_{\ell+1} - \tilde{\Theta}_\ell^\top \Delta \Theta_\ell \phi_\ell + \Delta \Theta_{\ell-1}\phi_{\ell-1} \right)\textrm{.}
\end{aligned}
\label{eq:neural_act_dkp_riform}
\end{equation}
From Eq.~\eqref{eq:neural_act_dkp_riform}, we can see how Eq.~\eqref{eq:new_na_after_dkp} can be reformulated as the original FE gradient with respect to the neural activity before the DKP weight update, denoted as $\nabla_{\phi_\ell} F$, plus terms resulting from the DKP weight update itself.

Note that, because of the forward-initialization of the network, we have $\phi_{\ell+1} = \Theta_\ell \phi_\ell$ for $\ell < L-1$, except at layer $L-1$, as $\phi_L$ has been clamped to the target label at the beginning of this stage. This implies that, for intermediate layers $\ell = 1$ to $L-1$, the error nodes $\epsilon_{\ell}$ are equal to zero for a single inference step. The DKP-PC neural activity update thus simplifies to
\begin{equation}
    \Delta \tilde{\phi}_{\ell} = \gamma \left( - \tilde{\Theta}_\ell^\top \Delta \Theta_\ell \phi_\ell + \Delta \Theta_{\ell-1} \phi_{\ell-1} \right)\textrm{.}
\label{eq:simplified_na}
\end{equation}
By now substituting the weight update term from Eq.~\eqref{eq:dkp_theta_update} in Eq.~\eqref{eq:simplified_na}, we obtain
\begin{equation}
\begin{aligned}
    \Delta \tilde{\phi}_{\ell} &= \gamma \left( - \Theta_\ell^\top \Delta \Theta_\ell \phi_\ell - \Delta \Theta_\ell^\top \Delta \Theta_\ell \phi_\ell + \Delta \Theta_{\ell-1} \phi_{\ell-1} \right)\\
    &= \alpha \gamma \left( \Theta_\ell^\top \Psi_{\ell+1} \delta_L \phi_\ell^\top \phi_\ell - \alpha \phi_\ell \delta_L^\top \Psi_{\ell+1}^\top \Psi_{\ell+1} \delta_L \phi_\ell^\top \phi_\ell - \Psi_\ell \delta_L \phi_{\ell-1}^\top \phi_{\ell-1} \right)\\
    &= \alpha \gamma \left( \|\phi_\ell\|_2^2 \Theta_\ell^\top \Psi_{\ell+1} \delta_L - \alpha \|\phi_\ell\|_2^2 \|\Psi_{\ell+1} \delta_L\|_2^2 \phi_\ell - \|\phi_{\ell-1}\|_2^2 \Psi_\ell \delta_L \right)\\
    &= \underbrace{\alpha \gamma \left( \|\phi_\ell\|_2^2 \Theta_\ell^\top \tilde{\delta}_{\ell+1} - \|\phi_{\ell-1}\|_2^2 \tilde{\delta}_{\ell} \right)}_{\textrm{Alignment term}} - \underbrace{\alpha^2 \gamma \left( \|\phi_\ell\|_2^2 \|\tilde{\delta}_{\ell+1}\|_2^2 \phi_\ell \right)}_{\textrm{Regularization term}}\label{eq:dkp-pc_neural_activity_update}\textrm{.}
\end{aligned}
\end{equation}
The first term in the neural activity update quantifies the discrepancy between two estimates of the error at layer $\ell$, namely the direct projection of the output error $\delta_L$ onto layer $\ell$, given by $\Psi_\ell \delta_L = \tilde{\delta}\ell$, and the backward projection produced by PC from layer $\ell+1$ through the error node, given by $\Theta\ell^\top \Psi_{\ell+1} \delta_L = \Theta_\ell^\top \tilde{\delta}_{\ell+1}$. Note that these projections are respectively weighted by the squared $\ell_2$ norm of neural activities $\phi_{\ell-1}$ and $\phi_{\ell}$. The second term of the update is an activity regularization term whose strength evolves dynamically throughout training and proportionally to $\|\phi_\ell\|^2$ and $\|\tilde{\delta}_{\ell+1}\|^2 = \|\Psi_{\ell+1} \delta_L\|^2$. Note that this specific interpretation holds only in the single-step inference phase regime, which is anyway the one considered in this paper.

\textit{3) Learning phase} -- Following the consecutive updates of the forward weights $\Theta_\ell$ and the neural activity $\phi_\ell$, the PC forward weight update, which aims to minimize the FE, is given by
\begin{equation}
\begin{aligned}
    \Delta \tilde{\Theta}_{\ell} &\propto \frac{\partial \tilde{F}}{\partial \tilde{\Theta}_\ell}\\
    &\propto \tilde{\epsilon}_{\ell+1} \phi_{\ell}^{*\top}\textrm{.}
\end{aligned}
\label{eq:theoretical_forward_update}
\end{equation}
While the unrolled mathematical expression of this update under the DKP regime results in a complex formulation that offers limited intuition, we empirically demonstrate in Appendix~\ref{app:dkp-pc-empirical} that the alignment term injected into the neural activity, and consequently into the subsequent weight update, is essential for achieving a more stable and stronger alignment between feedback and forward weights, thereby providing clearer insight into its role within DKP-PC. Finally, the feedback matrices $\Psi_\ell$ are updated as follows:

\begin{equation}
\Delta \Psi_\ell \propto \phi_{\ell}^* \tilde{\epsilon}_L^\top.
\label{eq:new_dkppc_feedback_update_rule}
\end{equation}
As for the previous case, although the explicit derivation of this update leads to a complex expression even in the linear case, our experiments in Appendix~\ref{app:dkp-pc-empirical} again confirm that the alignment component introduced through the neural activity update, and thus propagated into the feedback matrices, is crucial for obtaining better alignment between feedback and forward weights, demonstrating its contribution within DKP-PC.

Our analysis shows that the initial DKP forward-weight update induces non-zero errors $\epsilon_\ell$ across all layers, enabling the neural activities to be updated in a single inference step. This feature solves two key limitations of PC: the error delay by generating non-zero errors right away, and the error exponential decay, by directly projecting the output error $\delta_L$ to each layer through the feedback weights. Furthermore, the optimized neural activity $\phi_\ell^{*}$ from DKP-PC includes alignment information that is successively injected into the PC forward weights' update and the DKP feedback weights' update. This leads to faster convergence than PC or DKP individually, as it enables earlier and stronger error signals and improves the alignment between forward and feedback pathways, producing weight updates that more closely approximate those of BP.

\subsubsection{Empirical Analysis of Gradient Alignment}
\label{app:dkp-pc-empirical}
In this subsection, we empirically analyse the alignment between the forward weight gradients computed according to DKP-PC and standard DKP algorithms, in comparison with BP. The alignment is quantified using the cosine similarity $\cos(\theta) \in [-1,1]$, which measures the directional agreement between two arbitrary column vectors $u \in \mathbb{R}^{d \times 1}$ and $v \in \mathbb{R}^{d \times 1}$ as follows:
\begin{equation}
\label{eq:cosinesim}
\cos(\theta) = \frac{u^\top v}{\|u\|  \|v\|}= \frac{\sum_{i=1}^{n} u_i v_i}{\sqrt{\sum_{i=1}^{n} u_i^2} \, \sqrt{\sum_{i=1}^{n} v_i^2}},
\,
\end{equation}
where $\|\cdot\|$ denote the Euclidean norm. The cosine similarity reaches a maximum value of $1$ when the vectors are perfectly aligned ($\theta = 0$), a minimum value of $-1$ when they are diametrically opposed ($\theta = \pi$), and a value of $0$ when they are orthogonal ($\theta = \pi/2$), indicating no directional correlation. Figure \ref{fig:alignment-comparison} depicts the alignment across the nine layers of the VGG9-like CNN from Section~\ref{results} trained for 50 epochs on CIFAR-100. For both DKP and DKP-PC networks, the hyperparameters correspond to the best-performing configuration reported in Appendix~\ref{app:train_specs}. All curves consider instantaneous gradients, excluding weight decay and momentum contributions, and are smoothed using an exponential moving average with a window of 100 batches. The gradient for DKP-PC is computed by summing the gradients resulting from its DKP and PC phases. First, we analyse the alignment's behaviour of DKP-PC to that of standard DKP. Then, we analyse the effect of disabling the update of feedback and forward weights after computing the neural activity optimization, to empirically support our claims in the previous subsection.

\begin{figure}[!ht]
\begin{center}
\includegraphics[width=0.9\textwidth]{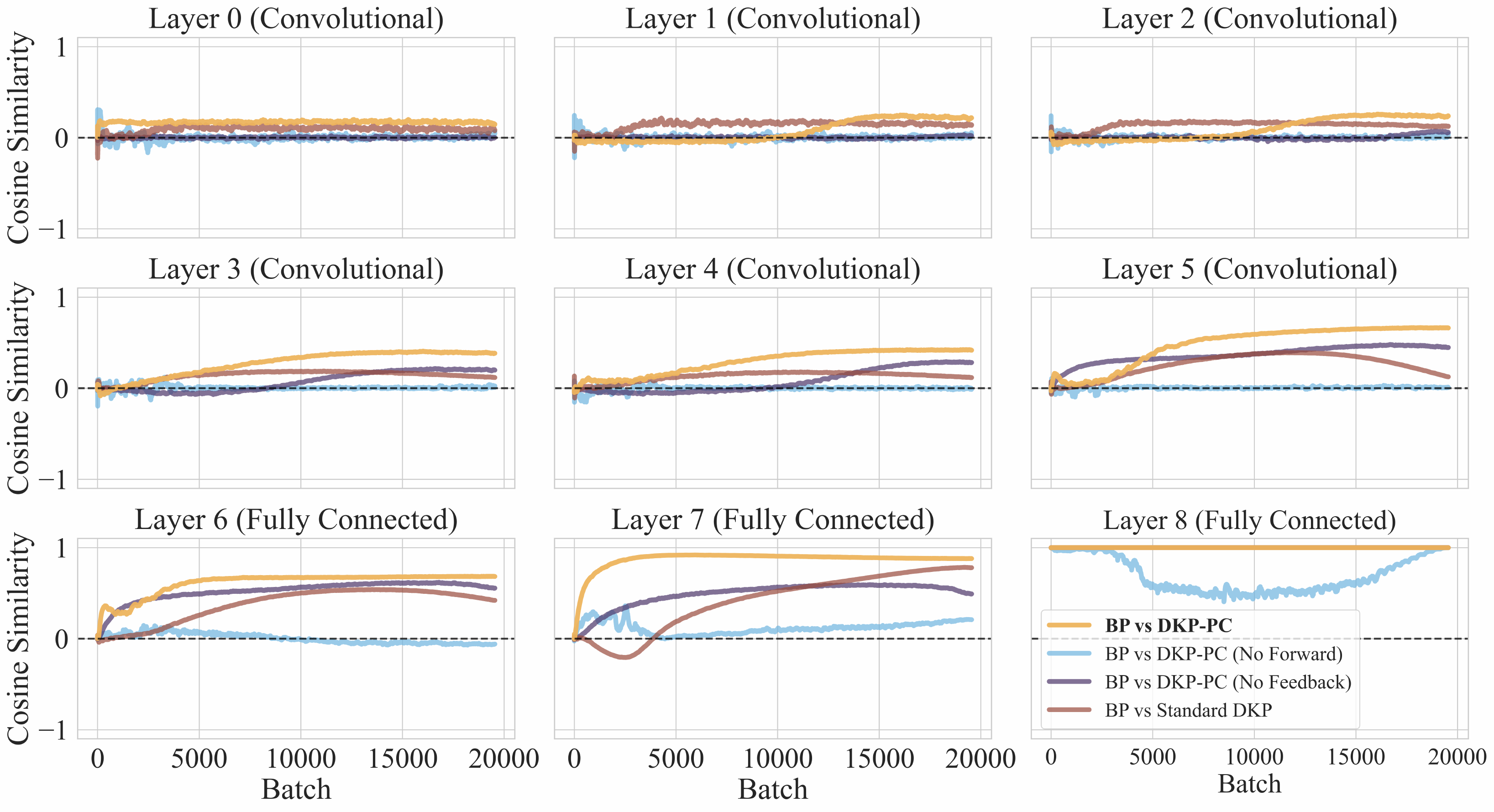}
\end{center}
\caption{Forward weight gradients alignment across layers of a VGG9-like CNN trained for 50 epochs on CIFAR-100. Each curve shows the cosine similarity between the instantaneous forward-weight gradient produced by DKP and DKP-PC algorithms, compared to the one computed with BP. All gradients exclude weight decay and momentum and are smoothed using an exponential moving average with a window of 100 batches. DKP (brown) displays positive but slow alignment with BP, progressively deteriorating with increasing distance from the output layer. DKP-PC (yellow) gradient is computed as sum of the gradients resulting from DKP and PC stages. It achieves consistently faster, higher, and more stable alignment across all layers compared to standard DKP. The light blue curve shows that disabling the PC forward-weight update in DKP-PC causes alignment to collapse in all layers, confirming its role in injecting alignment information into the forward weights. The blue curve, obtained by disabling the feedback-weight update in DKP-PC, demonstrates that the alignment and regularization terms introduced by the PC stage also improve the update of the feedback matrices, resulting in worse alignment when disabled.}
\label{fig:alignment-comparison}
\end{figure}

From Figure \ref{fig:alignment-comparison}, we can observe that, although standard DKP (brown curves) exhibits a positive gradient alignment with BP, its convergence is slow, and the alignment progressively deteriorates with increasing distance from the output layer, consistent with the analysis in Appendix \ref{app:dkp}. Several batches are required before reaching high cosine similarity, and alignment degradation is also observed at the end of training, particularly in layers 6 and 7, consistently with observations made in the FA literature \citep{refinetti2021align}.
In contrast, DKP-PC (yellow curves) achieves a faster and higher alignment with BP than standard DKP across all nine layers. DKP-PC indeed converges with fewer batches and exhibits better and more stable alignment throughout training. These observations not only align with the improved performance observed in the classification experiments but also support the theoretical analysis presented in the previous subsection. As shown in Eq.~\eqref{eq:dkp-pc_neural_activity_update}, the neural activity update incorporates both an alignment and a regularization terms, which contribute to the forward weight update in Eq.~\eqref{eq:theoretical_forward_update}. Consequently, this stage can be interpreted as a regularization factor applied to the preliminary DKP update in DKP-PC, which over time stabilizes alignment and compensates for the error introduced by the Moore–Penrose pseudoinverse (for both forward and feedback weights). To further support this claim, Figure \ref{fig:alignment-comparison} also includes two versions of DKP-PC in which the PC forward weight update and the DKP feedback weight update have respectively been ablated. For DKP-PC without PC forward weight update (light blue curves), alignment collapses in all hidden layers and deteriorates even in the output layer, highlighting the key role of injecting alignment information in the forward weights through the updated neural activity.
Moreover, the improvement in DKP-PC alignment also arises from the influence of the alignment and regularization terms in Eq.~\eqref{eq:dkp-pc_neural_activity_update} on the feedback matrix update via Eq.~\eqref{eq:new_dkppc_feedback_update_rule}. This yields a better-aligned update and again compensates for the distortion introduced by the Moore–Penrose pseudoinverse. This claim is empirically supported by the results for DKP-PC without DKP feedback weight update (blue curves). Consistent with expectations, alignment decreases relative to DKP-PC, exhibiting a slower and less effective value in every layer.

Through this analysis, we complement the preceding theoretical subsection by empirically demonstrating the synergy between the DKP and PC stages in DKP-PC. We show that DKP not only helps PC overcome two of its main limitations,  exponential error decay and error propagation delay, but that PC, in turn, acts as a regularizer for the DKP update, improving gradient alignment with BP across all layers and yielding a more effective learning algorithm. These results also reveal a possible alternative interpretation of DKP-PC: it may be viewed both as an acceleration mechanism for training PC networks and, alternatively, as a method to enhance alignment in feedback-alignment approaches. Furthermore, DKP-PC achieves these improvements without compromising performance or the locality of computations and updates, thereby preserving the hardware-friendly properties of both PC and DKP.

\subsubsection{Empirical analysis of the Inference phase}
\label{app:inference-phase}
In this section, we empirically analyse the inference phase of the DKP-PC algorithm and its incremental version, denoted iDKP-PC. DKP-PC performs the update of the forward weights, given by Eq.~\eqref{eq:weight_update_rule} at the end of the inference phase, after optimizing the neural activity following Eq.~\eqref{eq:neural_act_learn_rule}. In contrast, as introduced in \citet{salvatori2024stablefastfullyautomatic}, iPC performs an update of the forward weights
after every step of the neural activity optimization. The same concept can be applied to DKP-PC under a multiple–inference–steps regime (as opposed to the single-step regime considered in the main text of the paper). Indeed, once the forward weights are perturbed by the DKP update, they can be further optimized after each neural activity update, exactly as in iPC. The key difference is that an error signal is already available at the very beginning of the inference phase, both for the neural activity and for the forward weights.

Figure~\ref{fig:energy-evolution} compares the energy of four different four-layer MLP models trained on a Fashion-MNIST batch and averaged over 10 trials, while varying the magnitude of the neural activity learning rate. The curves depict the energy evolution during the inference phase for DKP-PC (blue), iDKP-PC (light blue), PC (brown), and iPC (yellow). Both DKP-PC and iDKP-PC start from a higher energy than PC and iPC, as mitigating the error-delay problem provides an error term at every layer, leading to larger energy values. In both cases, the incremental algorithms decrease toward a lower minimum than the non-incremental ones. This follows from the additional degree of freedom introduced by updating the forward parameters during inference. Interestingly, DKP-PC and iDKP-PC respectively converge to similar energy values than PC and iPC, suggesting that after several neural activity updates, the gradient term from PC in Eq.~\eqref{eq:neural_act_dkp_riform} dominates the dynamics, driving the networks toward similar low-energy regions. However, unlike standard PC, DKP-PC requires more optimization steps to reach these low energy values, likely due to both the larger error present at every layer and the presence of two distinct driving forces in the neural activity update: the standard PC gradient and the term implicitly introduced by DKP. Notably, the incremental version of DKP-PC demonstrates a convergence speed similar to that of standard iPC instead. This suggests that allowing forward parameters to update during inference compensates for the additional complexity introduced by DKP, effectively stabilizing the dynamics and enabling the network to exploit the immediate availability of error signals.

\begin{figure}[!ht]
\begin{center}
\includegraphics[width=0.9\textwidth]{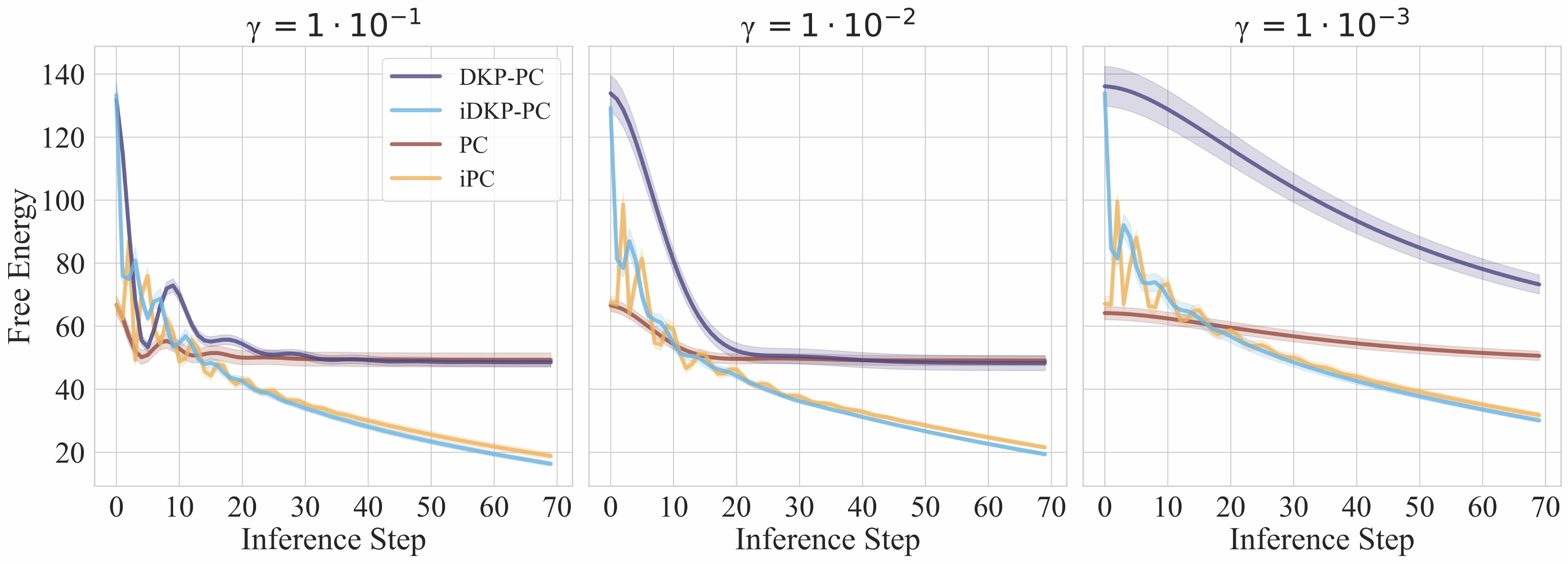}
\end{center}
\caption{Energy evolution of four-layer MLP networks on a Fashion-MNIST batch for three different neural activity learning rate magnitudes. DKP-PC and its incremental variant are shown in blue and light blue, respectively, while standard PC and iPC are represented in brown and yellow. Both DKP-PC variants start from higher energy values due to the immediate error term at every layer, and converge to levels similar to those of the standard PC and iPC networks. Although DKP-PC exhibits slower convergence due to the additional terms in the neural activity dynamics, its incremental version equals the convergence speed of iPC across all evaluated learning rates, suggesting that updating forward parameters during inference effectively compensates for the additional complexity introduced by DKP.}
\label{fig:energy-evolution}
\end{figure}

While in principle the neural activity optimization in PC can be run until full minimization of the FE, in practice it is performed for a finite number of steps \citep{pinchetti2024benchmarking}, typically exceeding the network’s depth to achieve the best results \citep{goemaere2025error}. Figure~\ref{fig:accuracy-evolution} presents the corresponding behaviour for DKP-PC and iDKP-PC by reporting their test accuracy as a function of the total number of inference steps, with boxplots showing the distribution across 30 trials. The evaluated networks are the same as for Figure~\ref{fig:energy-evolution}. Both algorithms exhibit a positive correlation between the number of neural activity optimization steps per batch and the final test accuracy, with higher numbers of inference steps yielding the best results in both cases. This also highlights a fundamental trade-off for DKP-PC networks, consistently with PC networks, between training time and performance. Nevertheless, as also noted in the main text, our proposed DKP-PC approach enables a better exploitation of the trade-off, as even a single inference step is sufficient to achieve results that surpass both baselines and advanced PC variants.

\begin{figure}[!ht]
\begin{center}
\includegraphics[width=0.9\textwidth]{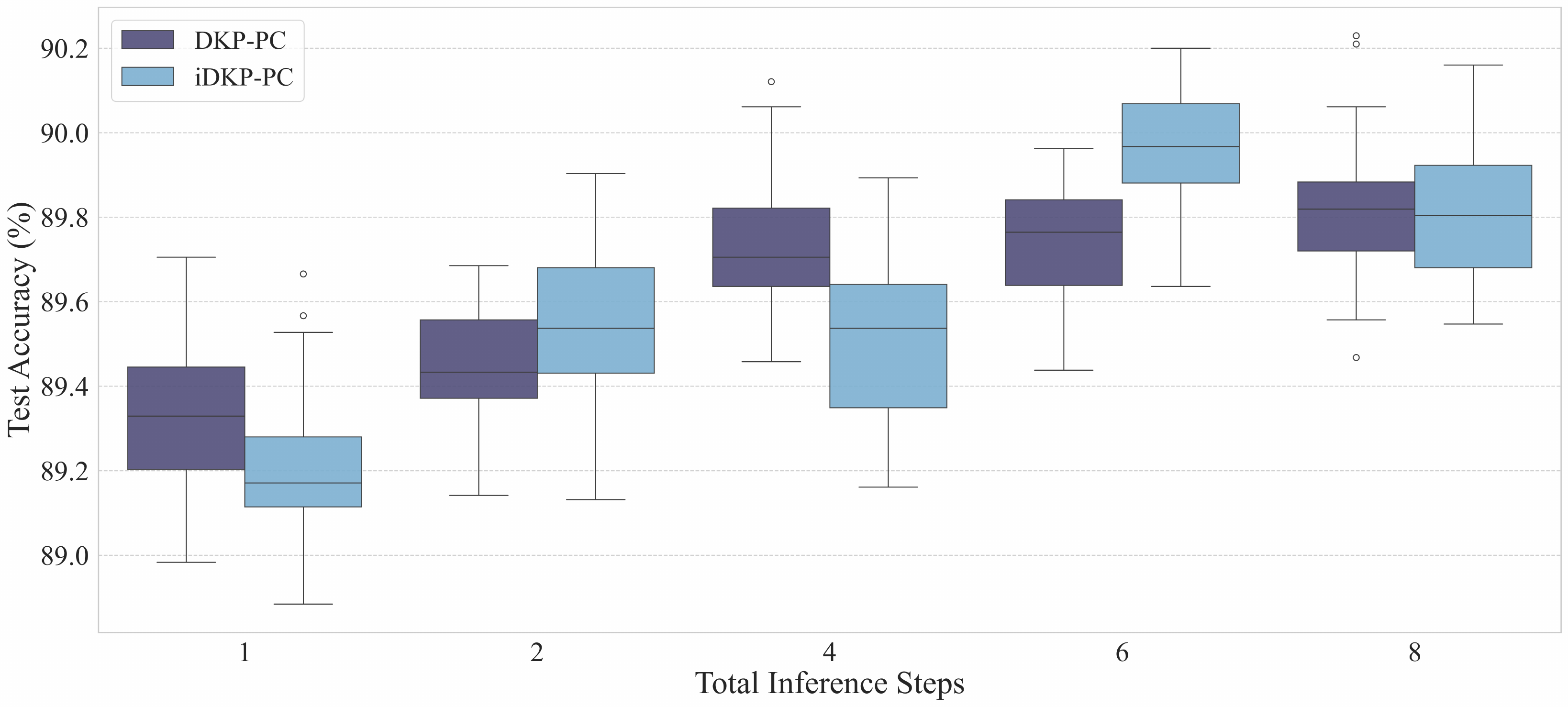}
\end{center}
\caption{Test accuracy distributions over 30 trials are shown as a function of the total number of neural activity optimization steps. Blue boxplots correspond to a four-layer MLP trained on Fashion-MNIST with DKP-PC, while light blue boxplots show the same architecture trained with iDKP-PC. In line with PC theory, both methods display a positive correlation between the number of optimization steps and the final test accuracy.}
\label{fig:accuracy-evolution}
\end{figure}

\subsection{Details of Classification Experiments}
\label{app:train_specs}
\textit{Models and datasets} -- The MLP models are evaluated on the MNIST and Fashion-MNIST datasets, which contain 28$\times$28 grayscale images and comprise 60k training and 10k test samples across 10 classes.
The CNN architectures are evaluated on CIFAR-10, CIFAR-100, and Tiny~ImageNet. CIFAR-10 and CIFAR-100 consist of 32$\times$32 RGB images with 50k training and 10k test samples for 10 and 100 classes, respectively. Tiny~ImageNet contains 200 classes with 100k training and 10k validation images of size 64$\times$64.
The MLPs use two hidden layers with 128 units each. The VGG-like CNNs include six convolutional layers and either one or three fully-connected layers, depending on the variant.
To ensure comparability with prior work and to support future benchmarking efforts, we adopt the same architectures as \citet{pinchetti2024benchmarking}. The full architectural specifications are reported in Table~\ref{tab:architectures}. For completeness and to ensure reproducibility, Table~\ref{tab:hyperparams} reports the optimal DKP-PC hyperparameters identified through our hyperparameter search, which were successively used to evaluate the model and obtain the results presented in Section~\ref{results}.

\begin{table}[h]
\vspace{1em}
\caption{Architectural details. FC size refers to the number of units in the fully-connected layers after the convolutional ones (if any).}
\vspace{1em}
\centering
\begin{tabular}{lccc}
\hline
 & \textbf{MLP} & \textbf{VGG-7} & \textbf{VGG-9} \\
\hline
Conv. channels & -- & {[}128x2, 256x2, 512x2{]} & {[}128x2, 256x2, 512x2{]} \\
Kernel sizes         & -- & {[}3, 3, 3, 3, 3, 3{]} & {[}3, 3, 3, 3, 3, 3{]} \\
Strides              & -- & {[}1, 1, 1, 1, 1, 1{]} & {[}1, 1, 1, 1, 1, 1{]} \\
Paddings             & -- & {[}1, 1, 1, 0, 1, 0{]} & {[}1, 1, 1, 1, 1, 1{]} \\
Pool window          & -- & $2 \times 2$ & $2 \times 2$ \\
Pool stride          & -- & 2 & 2 \\
FC size         & {[}128, 128, output{]} & {[}output{]} & {[}4096, 4096, output{]} \\
\hline
\end{tabular}
\label{tab:architectures}
\end{table}

\begin{table}[!ht]
\vspace{1em}
\caption{Hyperparameters employed for the DKP-PC networks in our experiments.}
\vspace{1em}
\centering
\resizebox{\linewidth}{!}{
\begin{tabular}{l|cc|cc|ccc}
\midrule
& \multicolumn{2}{c|}{\textbf{MLP}} & \multicolumn{2}{c|}{\textbf{VGG-7}} & \multicolumn{2}{c}{\textbf{VGG-9}} \\
& MNIST & FMNIST & CIFAR-10 & CIFAR-100 & CIFAR-10 & CIFAR-100 & Tiny ImageNet\\
\midrule
activation &  gelu&  gelu&  gelu&  tanh&  leaky&  leaky&  gelu\\
fw-lr &  $4.616 \mathrm{e}{-4}$&  $5.254 \mathrm{e}{-4}$&  $1.458 \mathrm{e}{-4}$&  $2.482 \mathrm{e}{-4}$&  $1.609\mathrm{e}{-4}$&  $1.602 \mathrm{e}{-4}$&  $7.373 \mathrm{e}{-5}$\\
fw-decay &  $3.737 \mathrm{e}{-2}$&  $2.744 \mathrm{e}{-5}$&  $3.626 \mathrm{e}{-4}$&  $9.664\mathrm{e}{-2}$&  $5.271\mathrm{e}{-2}$&  $1.040\mathrm{e}{-2}$& $2.893 \mathrm{e}{-5}$\\
fw-opt &  adamw&  adamw&  adam&  adam&  adam&  adam&  adamw\\
i-lr &  $1.068 \mathrm{e}{-3}$&  $8.297 \mathrm{e}{-1}$&  $5.655 \mathrm{e}{-2}$&  $1.036\mathrm{e}{-2}$&  $1.113\mathrm{e}{-3}$&  $1.169\mathrm{e}{-2}$& $3.136 \mathrm{e}{-3}$\\
i-mom &  0&  0&  0&  0&  0&  0& 0\\
i-steps &  1&  1&  1&  1&  1&  1& 1\\
fb-init &  ka-unif.&  ka-unif.&  orthog.&  ka-norm.&  ka-unif.&  xav-unif.&  orthog.\\
fb-lr &  $3.024 \mathrm{e}{-5}$&  $4.702 \mathrm{e}{-5}$&  $1.533 \mathrm{e}{-3}$&  $1.333\mathrm{e}{-3}$&  $1.664\mathrm{e}{-3}$&  $9.405\mathrm{e}{-4}$& $2.839 \mathrm{e}{-4}$\\
fb-decay &  $2.446 \mathrm{e}{-3}$&  $2.744 \mathrm{e}{-5}$&  $5.215 \mathrm{e}{-5}$&  $4.406\mathrm{e}{-5}$&  $1.099\mathrm{e}{-4}$& $1.040\mathrm{e}{-2}$& $4.656 \mathrm{e}{-5}$\\
fb-opt & adamw&  nadam&  adamw&  adamw&  adam&  nadam& adam\\
fb-gamma &  0.99975&  0.9995&  1&  0.99995& 0.9999&  0.9995& 0.99925\\
\midrule
\end{tabular}
\label{tab:hyperparams}
}
\end{table}

\textit{Training setup} -- Consistently with the work of \citet{pinchetti2024benchmarking}, MLPs are trained for 25 epochs and CNNs for 50 epochs, using a batch size of 128. Data augmentation is applied in the CNN experiments. For CIFAR-10 and CIFAR-100, images are randomly cropped to 32×32 pixels with 4-pixel padding during training. For Tiny ImageNet, random crops of 56×56 pixels are used during training without padding, while the test set is evaluated using centered crops of 56×56 pixels, also without padding. The forward weights' learning rate is updated using a warmup-cosine-annealing scheduler without restarts. The optimizers considered include Adam and AdamW. Feedback connections are trained with a separate optimizer than the forward weights, using an exponentially decaying learning rate updated per batch via an exponential learning rate scheduler, with the update parameter fb-gamma reported in Table~\ref{tab:hyperparams}. Different feedback initializations are explored during the hyperparameter search: Xavier-uniform/normal \citep{glorot2010understanding}, Kaiming-uniform/normal \citep{he2015delving}, and orthogonal \citep{saxe2013exact}. Feedback optimizers include Adam, AdamW, and Nadam \citep{adam2014method, dozat2016incorporating, loshchilov2017decoupled}. All additional experimental details are available on our GitHub repository [link will be included after double-blinded revision].

\subsection{Computational Trade-offs}
\label{app:computational_trade-offs}
Resource consumption experiments focus on latency and floating-point operations (FLOPs), evaluated on both an MLP and a CNN. For the MLP, experiments are conducted on a network with 256 units per layer using a single sample from MNIST, whereas for the CNN, a VGG-like model with 64-channel $3 \times 3$ convolutions is evaluated on a single sample from CIFAR-10. Latency is defined as the time required for a complete parameter update, including feedforward initialization and the update of forward and feedback matrices (when applicable). Computational cost is estimated by counting FLOPs in forward and backward passes, restricted to core MAC operations, with 1 MAC = 2 FLOPs\footnote{Variable contributions, such as activation function FLOPs, are excluded as they depend on the specific non-linearity. Accordingly, the reported FLOPs represent a lower-bound estimate of the actual computational cost}.  

The first row of Figure~\ref{fig:resource_comparison.png} illustrates the differences in training time, expressed in milliseconds, across the various models. Training time includes the forward pass and a complete backward pass, encompassing both forward and feedback weight updates, if applicable. The forward pass contribution is indicated by a red dashed line. The fastest algorithm is DKP, owing to the reduced dimensionality introduced in its update in Eq.~\eqref{eq:dkp-update}. BP is the second fastest algorithm, followed by DKP-PC. The latter, as well as the other local algorithms, have been evaluated in sequential mode, and therefore their parallelization potential is not considered in this section. Nevertheless, this already highlights the speed-up of DKP-PC compared to standard PC or iPC, as it consistently requires less time to fully update its parameters for both the MLP and CNN.  

The second row of Figure~\ref{fig:resource_comparison.png} compares the minimum FLOPs requirements across algorithms. DKP emerges as the most efficient method in architectures where hidden layers exceed the output layer in size. In this scenario, computing the intermediate error term $\delta_\ell$ only requires multiplying the output error $\delta_L \in \mathbb{R}^{d_{L}}$ by the random matrix $\Psi_\ell \in \mathbb{R}^{d_\ell \times d_{L}}$, which entails fewer MAC operations than BP. In contrast, BP requires multiplying $\Theta_{\ell+1}^\top \in \mathbb{R}^{d_{\ell+1} \times d_{\ell+2}}$ with the higher-layer error $\delta_{\ell+2} \in \mathbb{R}^{d_{\ell+2}}$, typically with $d_L < d_\ell$ for all $\ell \in \{0, \dots, L-1\}$. BP is the second most efficient algorithm overall, followed in order by DKP-PC, PC, and iPC. The logarithmic scale of the plot highlights the growth in computational complexity for PC and iPC as depth increases, since both the minimal number of inference-phase steps and the number of multiple matrix–vector multiplications per inference-phase step increase with depth. DKP-PC scales better than PC and iPC as it requires only one inference-phase step to match or surpass their accuracy, achieving nearly an order of magnitude fewer FLOPs, thereby underscoring its efficiency advantage.

\begin{figure}[!t]
\begin{center}
\includegraphics[width=0.9\textwidth]{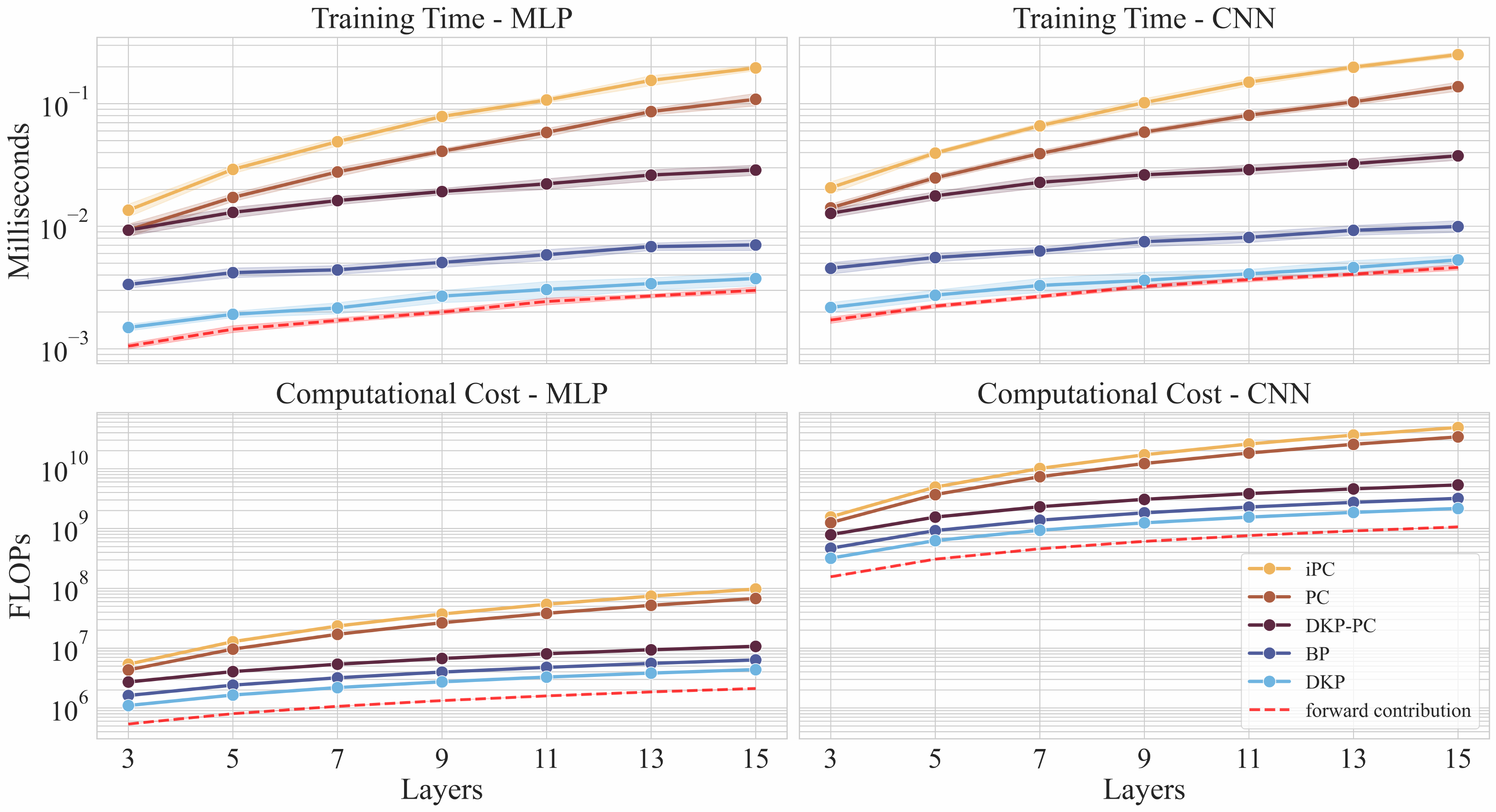}
\end{center}
\caption{The first row compares the training time on a logarithmic scale, measured as the sum of the forward pass and the complete parameter updates, with the contribution of the forward pass illustrated as a red dashed line. Results are averaged over 20 samples. The second row compares the minimum FLOPs requirements, estimated from the core MAC operations. In all plots, for both PC and iPC, the number of inference-phase steps is assumed to equal the network depth. In contrast, for DKP-PC, only a single step is considered, as it has been empirically demonstrated to be enough to achieve comparable performance.}
\label{fig:resource_comparison.png}
\end{figure} 

\subsection{Analysis of parallel execution}
\label{app:parallel}
While a fully-parallel implementation of PC is theoretically possible, it has so far been practically limited by the signal error delay and exponential decay problems \citep{zahid2023predictive, pinchetti2024benchmarking, goemaere2025error} detailed in Section \ref{errordelaydecay}. Here, we discuss how DKP-PC overcomes these limitations, yielding an algorithm capable of achieving lower training latency than BP.

Referring to Algorithm~\ref{alg:dkp-pc} and excluding the forward initialization, which incurs the same computational cost as BP, we first consider the \textit{Direct Feedback Alignment update} (1). Since each forward weight update depends only on the local neural activity and the error signal propagated from the output layer via the corresponding feedback matrix, all updates can be executed in parallel, reducing the time complexity of this phase from $\mathcal{O}(L)$ to $\mathcal{O}(1)$. In the subsequent \textit{Inference phase} (2), which is typically executed over multiple steps ($T \geq L$) in PC, we empirically demonstrate in Table~\ref{tab:acc_comparison} that a single step suffices to achieve accuracy comparable to or exceeding that of standard PC, reducing the time complexity from $\mathcal{O}(T)$ to $\mathcal{O}(1)$. Within this phase, updates of error neurons and neural activities are also parallelizable, as each relies solely on locally available information, resulting in $\mathcal{O}(1)$ time complexity. For completeness, we highlight that synchronization is still required, as error terms must be computed before updating neural activities. Successively, the \textit{Learning phase} (3) and \textit{Direct Kolen-Pollack update} (4) are executed sequentially, though they are independent and can be performed simultaneously. While these phases respectively have to iterate over all $L-1$ forward and backward parameter matrices, their computations are entirely local and hence can be executed in parallel across layers, reducing the time complexity from $\mathcal{O}(L)$ to $\mathcal{O}(1)$.

In summary, DKP-PC consists of four phases, each theoretically fully parallelizable with time complexity $\mathcal{O}(1)$. Each phase blocks the next, except for the final two, which may run concurrently. Consequently, DKP-PC’s time complexity does not grow with network depth, unlike BP, whose time complexity scales linearly as $\mathcal{O}(L)$. For sufficiently deep networks, we claim that the overall training time of DKP-PC will be significantly lower than that of BP. In practice, however, this advantage is challenged by existing hardware, which is heavily optimized for BP, and by the synchronization overhead inherent in software-based parallelization. These limitations could nevertheless be overcome by custom hardware accelerators designed to fully exploit DKP-PC’s parallelizability.
%%%%%%%%%%%%%%%%%%%%%%%%%%%%%%%%%%%%%%%%%%%%%%%%%%%%%%%%%%%%%%%%%%%%%%%%%%%%%%%
%%%%%%%%%%%%%%%%%%%%%%%%%%%%%%%%%%%%%%%%%%%%%%%%%%%%%%%%%%%%%%%%%%%%%%%%%%%%%%%

\end{document}